%% file: main.tex
\documentclass[conference]{IEEEtran}
\usepackage{times}

\usepackage[numbers]{natbib}
\usepackage{multicol}
\usepackage[bookmarks=true]{hyperref}
\usepackage{algs}

\input{includes}

\begin{document}

\title{Coherence-based Approximate Derivatives via \\ Web of Affine Spaces Optimization}

\author{\authorblockN{Daniel Rakita,
Chen Liang\authorrefmark{1},
Qian Wang\authorrefmark{1}}
\authorblockA{\small \authorrefmark{1}\textit{Equal Contribution}}
\authorblockA{Department of Computer Science, 
Yale University}
\authorblockA{\{daniel.rakita, dylan.liang, peter.wang.qw262\}@yale.edu}
}



%

\maketitle

\input{0-abstract}
\input{1-introduction}
\input{2-background}

\input{3-technical_overview}

\input{4-technical_details}
\input{5-algorithmic_details}

\input{6-evaluation1}

\input{7-evaluation2}
\input{8-evaluation3}
\input{9-discussion}

\IEEEpeerreviewmaketitle

\section*{Acknowledgments}
This work was supported by Office of Naval Research award N00014-24-1-2124


\bibliographystyle{plainnat}
\bibliography{references}

\input{10-appendix}

\end{document}

%% file: includes.tex
\usepackage{times}
\usepackage[numbers]{natbib}
\usepackage{multicol}
\usepackage[bookmarks=true]{hyperref}
\usepackage{lipsum}
\usepackage{comment}
\usepackage{balance}
\usepackage{graphicx}
\usepackage[table]{xcolor}
\usepackage{listings}
\usepackage{amsfonts}
\usepackage{amsmath}
\usepackage{amsthm}
\usepackage{tabularx}
\usepackage[linesnumbered,ruled,noend,vlined]{algorithm2e}

\SetCommentSty{mycommfont}

\usepackage{seqsplit}

\newcommand{\argmin}{\mathop{\mathrm{argmin}}}

\newtheorem{theorem}{Theorem}
\newtheorem{lemma}{Lemma}

\newtheorem{proposition}{Proposition}
\theoremstyle{definition}

%% file: 0-abstract.tex
\begin{abstract}
Computing derivatives is a crucial subroutine in computer science and related fields as it provides a local characterization of a function's steepest directions of ascent or descent.  In this work, we recognize that derivatives are often not computed in isolation; conversely, it is quite common to compute a \textit{sequence} of derivatives, each one somewhat related to the last.  Thus, we propose accelerating derivative computation by reusing information from previous, related calculations—a general strategy known as \textit{coherence}.  We introduce the first instantiation of this strategy through a novel approach called the Web of Affine Spaces (WASP) Optimization.  This approach provides an accurate approximation of a function's derivative object (i.e. gradient, Jacobian matrix, etc.) at the current input within a sequence.  Each derivative within the sequence only requires a small number of forward passes through the function (typically two), regardless of the number of function inputs and outputs.  We demonstrate the efficacy of our approach through several numerical experiments, comparing it with alternative derivative computation methods on benchmark functions.  We show that our method significantly improves the performance of derivative computation on small to medium-sized functions, i.e., functions with approximately fewer than 500 combined inputs and outputs.  Furthermore, we show that this method can be effectively applied in a robotics optimization context. We conclude with a discussion of the limitations and implications of our work.  Open-source code, visual explanations, and videos are located at the paper website: \href{https://apollo-lab-yale.github.io/25-RSS-WASP-website/}{https://apollo-lab-yale.github.io/25-RSS-WASP-website/}. 
\end{abstract}


%% file: 1-introduction.tex
\section{Introduction}
\label{sec:introduction}

Mathematical derivatives are fundamental to much of science. At a high level, derivatives offer a local characterization of a function's steepest ascent or descent directions. In practice, this property is frequently employed in numerical optimization, where derivatives guide the iterative process of navigating downhill through the landscape of a function \citep{nocedal1999numerical}.  For example, derivative-based optimization is widely used in robotics for tasks such as inverse kinematics, trajectory optimization, physics simulation, control, learning, and constrained planning.

Since derivative computation often takes place within a tight, low-level loop in the application stack, the speed of this process is critical to maintaining sufficient performance.  For example, consider a legged robot using a derivative-based model predictive control (MPC) algorithm to maintain balance \cite{di2018dynamic, kuindersma2016optimization}.  If the robot is nudged, it must compute derivatives very rapidly to guide the optimization process and allow the real-time reactive actuations of its legs to stay upright.

As we will discuss in \S\ref{sec:background}, there are several standard techniques to calculate the derivatives of a function \cite{griewank2008evaluating}.  These techniques generally involve repeatedly evaluating the function with slightly modified arguments, observing the resulting perturbations in the function's input or output space, and constructing the derivative from these observations.  However, since the number of function evaluations required typically scales with the number of inputs or outputs of the function, these approaches can quickly become prohibitively expensive when either, or especially both, of these dimensions increase.  

\begin{figure}
\centering
\includegraphics[width=\columnwidth]{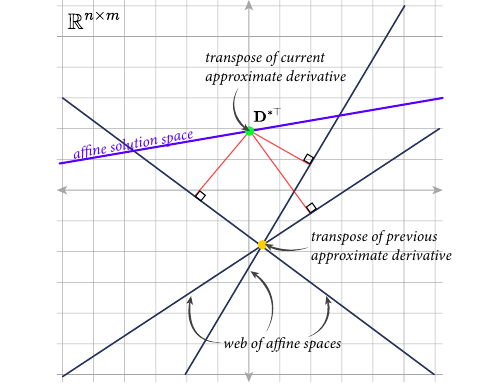}
\caption{ In this work, we present an approach for efficiently computing a sequence of approximate derivatives by reusing information from recent calculations.  Our approach first isolates an affine solution space where the true derivative must lie (purple line).  Next, a closed-form optimization procedure locates the point in this space that is the closest orthogonal distance (red lines) to a ``web'' of affine spaces (dark blue lines) that intersects at the previous approximate derivative (orange dot).  This optimal point will be the transpose of the approximate derivative matrix, $\mathbf{D}^{* \top}$ (green dot). }
\label{fig:teaser}
\vspace{-0.3cm}
\end{figure}

In this work, we recognize that derivatives are often not computed in isolation; conversely, it is quite common to calculate a \textit{sequence} of derivatives, each one building on the last.  For example, in optimization, function inputs typically change only slightly between iterations as small steps are taken downhill, with each input incrementally leading to the next.  The key insight of this work is that derivative computation can be accelerated by reusing information from previous, related calculations—a strategy known as \textit{coherence}.  

We present a first instantiation of this coherence-based strategy for derivative computation through a novel approach called the Web of Affine Spaces (WASP) Optimization. At its core, this approach frames derivative computation as a constrained least-squares minimization problem \cite{bjorck2024numerical}. Each iteration of the algorithm requires only one Jacobian-vector product (JVP) which creates an affine space within which the true derivative is guaranteed to lie. The optimization is then tasked with finding the transpose of an approximate derivative that lies on this affine space (specified by a hard constraint) while best aligning with previous, related computations (specified in the objective function).  This process is illustrated in Figure \ref{fig:teaser}.  

We provide a closed-form solution to this minimization problem by directly solving its corresponding Karush-Kuhn-Tucker (KKT) system \cite{fletcher2000practical, nocedal1999numerical}.  Our algorithm that uses this minimization also incorporates an error detection and correction mechanism that automatically identifies when its outputs drift too far from the ground-truth derivatives, allocating additional iterations to realign its results as needed.  This mechanism is guided by two user-adjustable parameters, affording a flexible balance between accuracy and computational performance tailored to specific applications.


The algorithm associated with our approach (\S\ref{sec:algorithmic_details}) is straightforward to implement and can be easily interfaced with existing code.  The algorithm does not require tape-based variable tracking or an external automatic differentiation library \cite{baydin2018automatic}; it simply uses standard forward passes through a function.  All ideas presented in this work can be implemented in less than 100 lines of code in any programming language that supports numerical computing.  


We demonstrate the effectiveness of our approach through a series of numerical experiments, benchmarking it against alternative derivative computation methods.  We show that our approach improves the performance of computing a sequence of derivatives on small to medium-sized functions, i.e., functions with approximately fewer than 500 combined inputs and outputs.  Additionally, we demonstrate its practical applicability in a robotics optimization context, showcasing its use in a Jacobian-pseudoinverse-based root-finding procedure to determine the pose of a quadruped robot with specified foot and end-effector placements.  We conclude with a discussion of the limitations and implications of our work.

%% file: 2-background.tex
\section{Background}
\label{sec:background}

In this section, we provide background for our approach, including notation, problem setup, standard approaches for derivative computation, and relevant prior work.

\subsection{Notation}
\label{sec:notation}

The main mathematical building blocks through this work are matrices and vectors.  Matrices will be denoted with bold upper case letters, e.g., $\mathbf{A}$, and vectors will be denoted with bold lower case letters, e.g., $\mathbf{x}$.  Indexing into matrices or vectors will use sub-brackets, e.g., $\mathbf{A}_{[0,1]}$ or $\mathbf{x}_{[2]}$.  A full row or column of a matrix can be referenced using a colon, e.g., $\mathbf{A}_{[:,0]}$ is the first column and $\mathbf{A}_{[0,:]}$ is the first row.  

\subsection{Problem Setup}

In this work, we will refer to some function under consideration as $f$, which has $n$ inputs and $m$ outputs, i.e., $f: \mathbb{R}^n \rightarrow \mathbb{R}^m$.  We will also assume that the given function $f$ is computable and differentiable.

The mathematical object we are trying to compute is the derivative object of $f$ at a given input $\mathbf{x}_k \in \mathbb{R}^n$, denoted as $\frac{\partial f}{\partial \mathbf{x}}\big|_{\mathbf{x}_k}$.  This derivative will be an $m \times n$ matrix, i.e., $\frac{\partial f}{\partial \mathbf{x}}\big|_{\mathbf{x}_k} \in \mathbb{R}^{m \times n}$, with the following structure:

\begin{equation}
\frac{\partial f}{\partial \mathbf{x}}\bigg|_{\mathbf{x}_k} = 
\begin{bmatrix} 
\frac{\partial f_{[1]}}{\partial \mathbf{x}_{k_{[1]}}} & \frac{\partial f_{[1]}}{\partial \mathbf{x}_{k_{[2]}}} & ... & \frac{\partial f_{[1]}}{\partial \mathbf{x}_{k_{[n]}}}  \\ 
\frac{\partial f_{[2]}}{\partial \mathbf{x}_{k_{[1]}}} & \frac{\partial f_{[2]}}{\partial \mathbf{x}_{k_{[2]}}} & ... & \frac{\partial f_{[2]}}{\partial \mathbf{x}_{k_{[n]}}}  \\ 
\vdots & \vdots & ... & \vdots \\ 
\frac{\partial f_{[m]}}{\partial \mathbf{x}_{k_{[1]}}} & \frac{\partial f_{[m]}}{\partial \mathbf{x}_{k_{[2]}}} & ... & \frac{\partial f_{[m]}}{\partial \mathbf{x}_{k_{[n]}}}  
\end{bmatrix}    
\end{equation}

This matrix is referred to as a \textit{Jacobian}, or specifically as a \textit{gradient} when $m = 1$ \cite{stewart2012calculus}. Throughout this work, however, we will consistently use the broader term, derivative.

\subsection{Problem Statement}

In this work, we are specifically looking to compute a \textit{sequence of approximate derivative matrices}:

$$..., \ \hat{\frac{\partial f}{\partial \mathbf{x}}}\bigg|_{\mathbf{x}_{k-1}}, \  \hat{\frac{\partial f}{\partial \mathbf{x}}}\bigg|_{\mathbf{x}_{k}}, \ \hat{\frac{\partial f}{\partial \mathbf{x}}}\bigg|_{\mathbf{x}_{k+1}}, \ ... $$

Our goal is to compute these approximate derivatives as quickly and as accurately as possible.  We assume that the derivative computation at an input $\mathbf{x}_k$ can utilize knowledge of all prior inputs and calculations (i.e., information available up to and including $k$), but it has no access to information about future inputs (i.e., data beyond $k$).  

There is an implicit assumption in our problem that adjacent inputs, e.g., $\mathbf{x}_k$ and $\mathbf{x}_{k+1}$, are relatively ``close'', as is often the case in iterative optimization.  However, our approach does not impose any specific requirement for the closeness of neighboring inputs.  Instead, it is informally assumed that the approach will perform more effectively when the inputs are closer to each other with efficiency or accuracy likely diminishing as the distance between inputs increases.


\subsection{Standard Derivative Computation Algorithms}

A common strategy for computing derivatives involves introducing small perturbations in the input or output space surrounding the derivative and incrementally constructing the derivative matrix by analyzing the local behavior exhibited by the derivative in response to these perturbations \citep{griewank2008evaluating}.    



Specifically, perturbing the derivative in the input space looks like the following:

\begin{equation}
\underbrace{\frac{\partial f}{\partial \mathbf{x}}\bigg|_{\mathbf{x}_k}}_{\text{derivative } \in \mathbb{R}^{m \times n}} \overbrace{\Delta \mathbf{x}}^{\text{tangent } \in \mathbb{R}^n} = \underbrace{\Delta \mathbf{f}}_{\text{Jacobian-vector product } \in \mathbb{R}^m}
\end{equation}

The $\Delta \mathbf{x}$ object here is commonly called a \textit{tangent}, and the resulting $\Delta \mathbf{f}$ is known as the \textit{Jacobian-vector product} (JVP) or \textit{directional derivative} \cite{baydin2018automatic, griewank2008evaluating}.  Conversely, perturbing the derivative in the output space looks like the following:

\begin{equation}
\underbrace{\Delta \mathbf{f}^\top}_{\text{adjoint } \in \mathbb{R}^{1 \times m}} \overbrace{\frac{\partial f}{\partial \mathbf{x}}\bigg|_{\mathbf{x}_k}}^{\text{derivative } \in \mathbb{R}^{m \times n}} = \underbrace{\Delta \mathbf{x}^\top}_{\text{vector-Jacobian Product } \in \mathbb{R}^{1 \times n}}
\end{equation}

The $\Delta \mathbf{f}^\top$ object here is commonly called an \textit{adjoint}, and the result $\Delta \mathbf{x}^\top$ is known as the \textit{vector-Jacobian product} (VJP) \cite{baydin2018automatic, griewank2008evaluating}.      

In general, there are two standard ways of computing JVPs: (1) forward-mode automatic differentiation; and (2) finite-differencing.  Forward-mode automatic differentiation propagates tangent information alongside standard numerical computations.  This technique commonly involves overloading floating-point operations to include additional tangent data \cite{margossian2019review}.  A JVP via first-order finite-differencing derives from the standard limit-based definition of a derivative:

\begin{equation}
    \label{eq:finite_diff}
    \frac{\partial f}{\partial \mathbf{x}}\bigg|_{\mathbf{x}_k} \ \Delta \mathbf{x} \approx \frac{f(\mathbf{x}_k + \epsilon\Delta \mathbf{x}) - f(\mathbf{x}_k)}{\epsilon}.
\end{equation}

If $\epsilon$ is small, this approximation is close to the true JVP.  Note that this JVP requires two forward passes through $f$, and additional JVPs at the same input $\mathbf{x}_k$ would only require one additional forward pass through $f$ each.     

Conversely, there is generally only one way of computing a VJP: reverse-mode automatic differentiation, often called \textit{backpropagation} in a machine-learning context \cite{baydin2018automatic, paszke2017automatic}.  This process involves building a computation graph (or Wengert List \cite{wengert1964simple}) on a forward pass, then doing a reverse pass over this computation graph to backward propagate adjoint information.  In general, VJP-based differentiation is more challenging to implement and manage compared to its JVP-based counterpart. This approach typically requires an external library to track variables and operations, enabling the construction of a computation graph \cite{jax2018github, paszke2019pytorch}. As a result, all downstream computations within a function must adhere to the same code structure or use the same library.     

Note that the concepts of JVPs and VJPs now give a clear strategy for isolating the whole derivative matrix.  For instance, using ``one-hot'' vectors for tangents or adjoints, i.e., vectors where only the $i$-th element is $1$ and all others are $0$, can effectively capture the $i$-th column or $i$-th row of the derivative matrix, respectively \cite{griewank2008evaluating}.  Thus, the derivative matrix can be fully recovered using $n$ JVPs or $m$ VJPs.  

In practice, a JVP-based approach is typically used if $n < m$ and a VJP-based approach is typically used if $m < n$.  However, as either $m$ or $n$ increase -- or, especially if both increase -- these approaches can quickly become inefficient.      

\subsection{Other Related Works}

Our work builds on previous methods aimed at accelerating derivative computations through approximations. For first-order derivatives, our approach is closest to Simultaneous Perturbation Stochastic Approximation (SPSA) \citep{spall1992multivariate}. SPSA, primarily used in optimization, estimates gradients using only two function evaluations by perturbing the function along one sampled direction.  This idea has inspired more recent work that approximates a gradient via a bundle of stochastic samples at a point \citep{baydin2022gradients, suh2022bundled}. 

Our approach also approximates derivatives by perturbing inputs in random directions and observing the output changes. Like SPSA (and related methods), it aims to achieve an approximate derivative with a small number forward passes through the function. However, we treat the random directions as a matrix in a least squares optimization, aligning the result with prior observations.  Also, unlike previous methods that primarily focus on gradients, our approach handles derivative objects of any shape, including full Jacobian matrices.

Interestingly, while coherence-based strategies have not been widely explored for first-order derivatives, they are frequently used in second-order derivative computations of scalar functions within optimization algorithms. For example, quasi-Newton methods like Davidon–Fletcher–Powell (DFP) \citep{fletcher1963rapidly}, Symmetric Rank 1 (SR1) \citep{davidon1991variable}, and Broyden–Fletcher–Goldfarb–Shanno (BFGS) \citep{fletcher1970new, broyden1970convergence, shanno1970conditioning} use the secant equation to iteratively build approximations of the Hessian matrix over a sequence of related inputs.  However, these algorithms cannot compute gradients directly, as they rely on them as inputs.  Through this lens, our current approach can be viewed as quite related to quasi-Newton methods for Hessian approximation, but specifically for first-order derivatives.



%% file: 3-technical_overview.tex
\section{Technical Overview}
\label{sec:technical_overview}

In this section, we overview the central concepts and intuitions of our idea.

\subsection{Differentiation as Linear System}

As covered in \S\ref{sec:background}, each pass through the function in JVP-based differentiation generates one JVP: $\frac{\partial f}{\partial \mathbf{x}}\big|_{\mathbf{x}_k} \Delta \mathbf{x} = \Delta \mathbf{f}$.  We can bundle several tangents together in a matrix in order to get a matrix of JVPs:

\begin{equation}
\begin{gathered}
\frac{\partial f}{\partial \mathbf{x}} \bigg|_{\mathbf{x}_k} 
\underbrace{\begin{bmatrix}\begin{array}{c|c|c|c} 
\Delta\mathbf{x}_1 & \Delta\mathbf{x}_2 & \dots & \Delta\mathbf{x}_r
\end{array}\end{bmatrix}}_{\Delta\mathbf{X}} = \\
\underbrace{\begin{bmatrix} \begin{array}{c|c|c|c} 
\Delta\mathbf{f}_1 & \Delta\mathbf{f}_2 & \dots & \Delta\mathbf{f}_r 
\end{array}\end{bmatrix}}_{\Delta\mathbf{F}}
\end{gathered}
\end{equation}

\begin{equation}
\label{eq:bundle}
\underbrace{\frac{\partial f}{\partial \mathbf{x}} \bigg|_{\mathbf{x}_k} }_{m \times n} \ \underbrace{\Delta\mathbf{X}}_{n \times r} = \underbrace{\Delta\mathbf{F}}_{m \times r}.
\end{equation}

Note that the $i$-th tangent vector and JVP are denoted as $\Delta \mathbf{x}_i$ and $\Delta \mathbf{f}_i$, respectively.  We use this same notation throughout the paper.  From Equation \ref{eq:bundle}, we see that JVP-based differentiation can also be interpreted as setting $\Delta\mathbf{X}$ to be the identity matrix and ``solving'' a linear system:

\begin{equation}
\frac{\partial f}{\partial \mathbf{x}}\bigg|_{\mathbf{x}_k} \ \mathbf{I} = \Delta\mathbf{F}  \ \Longrightarrow \ \frac{\partial f}{\partial \mathbf{x}}\bigg|_{\mathbf{x}_k} = \Delta\mathbf{F}
\end{equation}

This concept is mathematically straight forward, but it is important to note that generating $\Delta \mathbf{F}$ can be computationally expensive as it requires $n$ forward passes through $f$.  Our idea builds on this linear system idea, with $\Delta \mathbf{X}$ no longer only being an identity matrix.

\subsection{Differentiation as Least Squares Optimization}

In the section above, we assessed the linear system $\frac{\partial f}{\partial \mathbf{x}}\big|_{\mathbf{x}_k} \Delta \mathbf{X} = \Delta \mathbf{F}$.  If we take the transpose of both sides, we get the following: 

$$\Delta \mathbf{X}^\top \frac{\partial f}{\partial \mathbf{x}}\bigg|_{\mathbf{x}_k}^\top = \Delta \mathbf{F}^\top.$$  

This equation now nicely matches a standard ``$Ax = b$'' linear system, with ``$x$'' being an unknown matrix variable, in our case.  We cast this linear system as a least squares optimization \cite{bjorck2024numerical}:

\begin{equation}
\label{eq:least_squares}
\frac{\partial f}{\partial \mathbf{x}}\bigg|_{\mathbf{x}_k}^\top = \argmin_{\mathbf{D}^\top}||\Delta\mathbf{X}^\top \mathbf{D}^\top - \Delta\mathbf{F}^\top||^2_F.
\end{equation}

\noindent Here, $\mathbf{D}^\top$ is the decision variable acting as the derivative matrix, and $F$ denotes the Frobenius norm over matrices.  This formulation offers a clear analytical framework for considering general solutions, even when $\Delta \mathbf{X}$ is not a square or identity matrix.  The closed-form solution for this optimization problem is the following \cite{golub2013matrix}:

\begin{equation}
    \label{eq:least_squares_solution}
    \mathbf{D}^\top = (\Delta\mathbf{X}^\top)^\dagger \Delta\mathbf{F}^\top \Longrightarrow \mathbf{D} = \Delta \mathbf{F} \ \Delta \mathbf{X}^\dagger
\end{equation}

\noindent where the $\dagger$ symbol denotes the Moore-Penrose Pseudoinverse \cite{strang2022introduction}.  If $\Delta \mathbf{X}$ is full rank, $r \geq n$ (i.e., $\Delta \mathbf{X}^\top$ is square or tall), and $\Delta\mathbf{F}$ is a matrix of JVPs corresponding to the tangents in $\Delta \mathbf{X}$, this solution exactly matches the true derivative matrix.  However, this solution has not yet improved efficiency since we would still have to compute $\Delta\mathbf{F}$, which was the most expensive step from before.   

A key insight in this work is that $\Delta\mathbf{F}$ in the minimization above can be replaced with an approximation, $\hat{\Delta\mathbf{F}}$:

\begin{equation}
\label{eq:least_squares_with_approximation}
\frac{\partial f}{\partial \mathbf{x}}\bigg|_{\mathbf{x}_k}^\top  \approx \argmin_{\mathbf{D}^\top}||\Delta\mathbf{X}^\top \mathbf{D}^\top - \hat{\Delta\mathbf{F}}^\top||^2_F.
\end{equation}

\noindent Rather than fully recomputing $\Delta \mathbf{F}$ for each new input, we \textit{incrementally update} $\hat{\Delta \mathbf{F}}$ \textit{across a sequence of inputs}. Since $\hat{\Delta \mathbf{F}}$ is an approximation, the entire minimization process now becomes an approximation as well. Through the remainder of this work, we argue that, given an additional constraint on Equation \ref{eq:least_squares_with_approximation} and a particular strategy for updating $\hat{\Delta \mathbf{F}}$, this approach is more efficient than standard approaches for computing a sequence of derivative matrices while maintaining accuracy sufficient for practical use.

%% file: 4-technical_details.tex
\section{Technical Details}
\label{sec:technical_details}

In this section, we detail the Web of Affine of Spaces (WASP) Optimization approach for computing a sequence of approximate derivatives matrices.

\subsection{Affine Solution Space}

As discussed above, the bottleneck of Equation \ref{eq:least_squares} is the calculation of $\Delta\mathbf{F}$, the matrix bundle consisting of $r$ JVPs, where $r \geq n$. Rather than relying solely on $r$ JVPs, we first think about how much information about the derivative solution can be inferred from only one JVP. 

Suppose we have one fixed tangent of random values, $\Delta\mathbf{x} \in \mathbb{R}^n$ with a corresponding JVP, $\Delta \mathbf{f} \in \mathbb{R}^m$.  We can plug these vectors into Equation \ref{eq:least_squares}, with $\Delta \mathbf{X}^\top \equiv \Delta \mathbf{x}^\top \in \mathbb{R}^{1 \times n}$ and $\Delta \mathbf{F}^\top \equiv \Delta \mathbf{f}^\top \in \mathbb{R}^{1 \times m}$.  If $n > 1$, we have $\Delta \mathbf{X}^\top$ as a ``wide matrix'' that elicits an under-determined least squares system with \textit{infinitely} many solutions \cite{strang2022introduction}.  The space of all solutions is parameterized as follows:

\begin{equation}
\label{eq:affine_solution_space}
(\Delta \mathbf{x}^\top)^\dagger \Delta \mathbf{f}^\top + \mathbf{Z}_{\Delta \mathbf{x}^\top}\mathbf{Y}
\end{equation}

\noindent where $\mathbf{Z}_{\Delta \mathbf{x}^\top} \in \mathbb{R}^{n \times (n-1)}$ is the null space matrix of $\Delta \mathbf{x}^\top$ (i.e., $\Delta \mathbf{x}^\top \mathbf{Z}_{\Delta \mathbf{x}^\top} = \mathbf{0}$), and $\mathbf{Y}$ is any matrix in $\mathbb{R}^{(n-1) \times m}$.  Equation \ref{eq:affine_solution_space} defines an $(n-1) \times m$-dimensional \textit{affine space} encompassing all possible solutions, where $\mathbf{Z}_{\Delta \mathbf{x}^\top}\mathbf{Y}$ represents the associated vector space, and $(\Delta \mathbf{x}^\top)^\dagger \Delta \mathbf{f}^\top$ (the minimum-norm solution) serves as the offset from the origin.  Even with just one JVP, we have captured a space where the true derivative must lie for some setting of $\mathbf{Y}$.  Our idea, covered below, constrains the solution to lie within the space defined in Equation \ref{eq:affine_solution_space}, while also maintaining alignment with recent computations.

\subsection{Web of Affine Spaces}
\label{sec:web_of_affine_spaces}

In the previous section, we isolated a space where the true derivative must lie given only a single JVP.  We now assess what happens if we consider more JVPs.  

Consider $r$ JVPs, $\Delta \mathbf{f}_i$, with corresponding tangents $\Delta \mathbf{x}_i$, where $i \in {1, ..., r}$. Each JVP defines its own affine solution space, within which the true derivative must reside: $(\Delta \mathbf{x}_i^\top)^\dagger \Delta \mathbf{f}_i^\top + \mathbf{Z}_{\Delta \mathbf{x}_i^\top}\mathbf{Y}$. Knowing that the true derivative must lie within \textit{all} of these affine spaces, we can deduce that it must be located at the \textit{intersection} of these spaces. Indeed, when $r \geq n$, the intersection of all these spaces results in a single, unique matrix: the exact derivative. Essentially, this is precisely what the solution in Equation \ref{eq:least_squares_solution} achieves.  

We now return to the idea presented in Equation \ref{eq:least_squares_with_approximation}.  What if some JVPs are approximate, $\hat{\Delta \mathbf{f}}$, rather than ground-truth, $\Delta \mathbf{f}$? The idea here is that, if we have one ground-truth JVP, $\Delta \mathbf{f}_i$, along with at least $n-1$ other approximate JVPs, $\hat{\Delta \mathbf{f}}_j$, we can still force the solution to lie on $(\Delta \mathbf{x}_i^\top)^\dagger \Delta \mathbf{f}_i^\top + \mathbf{Z}_{\Delta \mathbf{x}_i^\top}\mathbf{Y}$ in a manner that gets as close as possible to the affine spaces corresponding to the other approximate JVPs, $(\Delta \mathbf{x}_j^\top)^\dagger \hat{\Delta \mathbf{f}}_j^\top + \mathbf{Z}_{\Delta \mathbf{x}_j^\top}\mathbf{Y}$. We refer to these approximate JVP affine spaces as the \textit{web of affine spaces}.  


\subsection{Web of Affine Spaces Optimization}

In this section, we overview the mathematical side of the Web of Affine Spaces (WASP) Optimization.  We specify the algorithmic details that instantiate this math in practice in \S\ref{sec:algorithmic_details}.  

Suppose $\Delta \mathbf{X}$ is a full rank $n \times r$ matrix where $r \geq n$.  We consider the columns of $\Delta \mathbf{X}$ to be $r$ separate tangent vectors where the $i$-th column is denoted as $\Delta \mathbf{x}_i$.  Assume we have a current input, $\mathbf{x}_k$, a selected tangent vector, $\Delta \mathbf{x}_i$, and a JVP corresponding to $\mathbf{x}_k$ in the $\Delta \mathbf{x}_i$ direction, $\Delta \mathbf{f}_i$ (likely computed using Equation \ref{eq:finite_diff}).  Also, assume we have a web of affine spaces matrix, $\hat{\Delta \mathbf{F}}$.    




We cast the optimization described above as a modified version of Equation \ref{eq:least_squares_with_approximation} with an added constraint:

\begin{equation}
\label{eq:wasp}
\begin{gathered}
\frac{\partial f}{\partial \mathbf{x}}\bigg|_{\mathbf{x}_k}^\top  \approx \argmin_{\mathbf{D}^\top}||\Delta\mathbf{X}^\top \mathbf{D}^\top - \hat{\Delta\mathbf{F}}^\top||^2_F, \\  
s.t. \ \ \Delta \mathbf{x}_i^\top \mathbf{D}^\top = \Delta\mathbf{ \mathbf{f}}_i^\top
\end{gathered}
\end{equation}

\noindent This constrained optimization best matches the intersection of the web of affine spaces, specified in the objective function, while also restricting the solution to lie on the affine solution space, specified in the constraint.  

The solution to Equation \ref{eq:wasp} is the following:

\begin{equation}
\label{eq:solution}
\begin{bmatrix} \mathbf{D}^{* \top} \\ \mathbf{\Lambda}^{* \top}\end{bmatrix} = \begin{bmatrix} 
2\Delta\mathbf{X} \Delta\mathbf{X}^\top & -\Delta\mathbf{x}_i \\ 
\Delta\mathbf{x}_i^\top & 0
\end{bmatrix}^{-1}\begin{bmatrix} 2\Delta\mathbf{X} \hat{\Delta\mathbf{F}}^\top\\ \Delta \mathbf{f}_i^\top\end{bmatrix}.
\end{equation}

\noindent This solution is derived using a Karush-Kuhn-Tucker (KKT) system, as seen in detail in \S\ref{sec:proof}.  Here, $\mathbf{D}^{* \top} \in \mathbb{R}^{n \times m}$ is the transpose of the approximate derivative matrix and $\mathbf{\Lambda}^{* \top} \in \mathbb{R}^{1 \times m}$ is a vector of Lagrange multipliers.  We can rewrite the solution in Equation \ref{eq:solution} by taking the block matrix inverse of the left matrix \cite{horn2012matrix}:

\begin{equation}
\label{eq:solution_block}
\begin{gathered}
\begin{bmatrix} \mathbf{D}^{* \top} \\ \mathbf{\Lambda}^{* \top}\end{bmatrix} = 
\begin{bmatrix} 
\mathbf{A}^{-1} - \mathbf{A}^{-1}\Delta\mathbf{x}_is_i^{-1} \Delta \mathbf{x}_i^\top \mathbf{A}^{-1} & \mathbf{A}^{-1}\Delta \mathbf{x}_is_i^{-1} \\ 
-s_i^{-1}\Delta\mathbf{x}_i^\top \mathbf{A}^{-1} & s_i^{-1}
\end{bmatrix}\\ \begin{bmatrix} 2\Delta\mathbf{X} \hat{\Delta\mathbf{F}}^\top\\ \Delta \mathbf{f}_i^\top\end{bmatrix} \\
\mathbf{A} = 2\Delta\mathbf{X} \Delta\mathbf{X}^\top, \ s_i = \Delta \mathbf{x}_i^{\top}\mathbf{A}^{-1}\Delta \mathbf{x}_i.
\end{gathered}
\end{equation}

\noindent Here, $s_i$ is the Schur Complement of the block matrix $\mathbf{A} = 2\Delta\mathbf{X} \Delta\mathbf{X}^\top$.  

Because we do not use the Lagrange multipliers in this work, we can use Equation \ref{eq:solution_block} to more directly solve for $\mathbf{D}^{* \top}$:


\begin{equation}
\footnotesize
\label{eq:solution_d}
\begin{gathered}
\mathbf{D}^{* \top} = \mathbf{A}^{-1}(\mathbf{I}_{n \times n} - s_i^{-1}\Delta\mathbf{x}_i \Delta \mathbf{x}_i^\top \mathbf{A}^{-1})2\Delta\mathbf{X} \hat{\Delta\mathbf{F}}^\top +  \\ s_i^{-1}\mathbf{A}^{-1}\Delta \mathbf{x}_i\Delta \mathbf{f}_i^\top.   
\\
\end{gathered}
\end{equation}

\noindent Here, $\mathbf{I}_{n \times n}$ is an $n \times n$ identity matrix, and $\mathbf{A}$ and $s$ can be found in Equation \ref{eq:solution_block}.  For interested readers, we discuss the geometric significance of Equation \ref{eq:solution_d} in the Appendix \S\ref{sec:geometric_interpretation_of_solution}.  

Again, this section has only covered the mathematical procedure behind the Web of Affine Spaces Optimization.  In \S\ref{sec:algorithmic_details}, we overview our algorithm that utilizes this optimization, including how to initialize and update $\hat{\Delta \mathbf{F}}$, how to preprocess and cache parts of Equation \ref{eq:solution_d} to accelerate the optimization at runtime, how to determine if $\mathbf{D}^{* \top}$ is an acceptable approximation, and how to achieve a more accurate result if the expected error is too high.

\subsection{Derivation of Web of Affine Spaces Solution}
\label{sec:proof}

In this section, we derive the solution specified in Equation \ref{eq:solution}.  First, note that another way of writing the objective function $||\Delta\mathbf{X}^\top \mathbf{D}^\top - \hat{\Delta\mathbf{F}}^\top||^2_F$ is the following: 

$$
 \ tr( \ (\Delta\mathbf{X}^\top \mathbf{D}^\top - \hat{\Delta\mathbf{F}}^\top)^\top (\Delta\mathbf{X}^\top \mathbf{D}^\top - \hat{\Delta\mathbf{F}}^\top) \ ),
$$

\noindent where $tr$ is the matrix trace.  We now multiply the terms within the trace function:

$$
\begin{gathered}
tr( \ (\mathbf{D}\Delta \mathbf{X} - \hat{\Delta \mathbf{F}}) ( \Delta\mathbf{X}^\top \mathbf{D}^\top - \hat{\Delta\mathbf{F}}^\top) \ ) = \\
tr( \ \mathbf{D}\Delta\mathbf{X}\Delta\mathbf{X}^\top\mathbf{D}^\top - 2\hat{\Delta \mathbf{F}}\Delta\mathbf{X}^\top\mathbf{D}^\top + \hat{\Delta \mathbf{F}}\hat{\Delta \mathbf{F}}^\top \ ).
\end{gathered}
$$

Writing the whole optimization out in this form:

$$
\begin{gathered}
\argmin_{\mathbf{D}^\top} \ tr( \mathbf{D}\Delta\mathbf{X}\Delta\mathbf{X}^\top\mathbf{D}^\top - 2\hat{\Delta \mathbf{F}}\Delta\mathbf{X}^\top\mathbf{D}^\top + \hat{\Delta \mathbf{F}}\hat{\Delta \mathbf{F}}^\top ), \\
s.t. \ \ \Delta\mathbf{x}_i^\top \mathbf{D}^\top = \Delta\mathbf{f}_i^\top   
\end{gathered}
$$

\noindent we form the Lagrangian of the optimization:

$$
\begin{gathered}
\mathcal{L}(\mathbf{D}^\top, \mathbf{\Lambda}^\top) = tr( \ \mathbf{D}\Delta\mathbf{X}\Delta\mathbf{X}^\top\mathbf{D}^\top - 2\hat{\Delta \mathbf{F}}\Delta\mathbf{X}^\top\mathbf{D}^\top + \\ \hat{\Delta \mathbf{F}}\hat{\Delta \mathbf{F}}^\top \ ) - (\Delta\mathbf{x}_i^\top \mathbf{D}^\top - \Delta\mathbf{f}_i^\top)\mathbf{\Lambda}
\end{gathered}
$$

\noindent where $\mathbf{\Lambda} \in \mathbb{R}^{m \times 1} $ are the Lagrange multipliers.  

A first-order necessary condition for an optimal solution is that the Karush-Kuhn-Tucker (KKT) conditions are satisfied \cite{nocedal1999numerical}. Specifically, for an equality constrained problem, this means that the partial derivatives of the Lagrangian with respect to both the decision variables and the Lagrange multipliers (associated with the equality constraints) are zero: 

$$
\frac{\partial \mathcal{L}}{\partial \mathbf{D}^\top} = 0, \ \  \frac{\partial \mathcal{L}}{\partial \mathbf{\Lambda}^\top} = 0.
$$

We start with the first requirement:

$$
\begin{gathered}
\frac{\partial \mathcal{L}}{\partial \mathbf{D}^\top} = (\frac{\partial }{\partial \mathbf{D}^\top} \ \mathbf{D}\Delta\mathbf{X}\Delta\mathbf{X}^\top\mathbf{D}^\top - 2\hat{\Delta \mathbf{F}}\Delta\mathbf{X}^\top\mathbf{D}^\top + \\  \hat{\Delta \mathbf{F}}\hat{\Delta \mathbf{F}}^\top)^\top - \frac{\partial }{\partial \mathbf{D}^\top} \ \Delta\mathbf{x}_i^\top\mathbf{D}^\top\mathbf{\Lambda} = \\
(2\mathbf{D}\Delta\mathbf{X}\Delta\mathbf{X}^\top - 2\hat{\Delta \mathbf{F}}\Delta\mathbf{X}^\top)^\top - \Delta\mathbf{x}_i\mathbf{\Lambda}^\top = \\
2\Delta\mathbf{X}\Delta\mathbf{X}^\top\mathbf{D}^\top - 2\Delta\mathbf{X}\hat{\Delta \mathbf{F}}^\top - \Delta\mathbf{x}_i\mathbf{\Lambda}^\top.
\end{gathered}
$$

We now set the term equal to zero:

$$
\begin{gathered}
2\Delta\mathbf{X}\Delta\mathbf{X}^\top\mathbf{D}^\top - 2\Delta\mathbf{X}\hat{\Delta \mathbf{F}}^\top - \Delta\mathbf{x}_i\mathbf{\Lambda}^\top = 0 \\ 
\Longrightarrow 2\Delta\mathbf{X}\Delta\mathbf{X}^\top\mathbf{D}^\top - \Delta\mathbf{x}_i\mathbf{\Lambda}^\top = 2\Delta\mathbf{X}\hat{\Delta \mathbf{F}}^\top.
\end{gathered}
$$

For the second requirement from the Lagrangian, $\frac{\partial \mathcal{L}}{\partial \mathbf{\Lambda}^\top} = 0$, we have the following:

$$
\begin{gathered}
\Delta\mathbf{x}_i^\top \mathbf{D}^\top - \Delta\mathbf{f}_i^\top = 0 \\ 
\Longrightarrow \Delta\mathbf{x}_i^\top \mathbf{D}^\top = \Delta\mathbf{f}_i^\top.
\end{gathered}
$$

Putting the previous components together into a KKT system, we have the following matrix equation:

$$
\begin{bmatrix} 2\Delta\mathbf{X}\Delta\mathbf{X}^\top & -\Delta \mathbf{x}_i \\ 
\Delta\mathbf{x}_i^\top & 0
\end{bmatrix} 
\begin{bmatrix} 
\mathbf{D}^{*\top} \\ \mathbf{\Lambda}^{*\top}
\end{bmatrix} = \begin{bmatrix} 2\Delta\mathbf{X}\hat{\Delta \mathbf{F}}^\top \\ \Delta\mathbf{f}_i^\top\end{bmatrix}.
$$

Using the matrix inverse, our final solution is the following:

$$
\begin{bmatrix} 
\mathbf{D}^{*\top} \\ \mathbf{\Lambda}^{*\top}
\end{bmatrix} = \begin{bmatrix} 2\Delta\mathbf{X}\Delta\mathbf{X}^\top & -\Delta \mathbf{x}_i \\ 
\Delta\mathbf{x}_i^\top & 0
\end{bmatrix}^{-1} \begin{bmatrix} 2\Delta\mathbf{X}\hat{\Delta \mathbf{F}}^\top \\ \Delta\mathbf{f}_i^\top\end{bmatrix},
$$

\noindent matching the solution seen in Equation \ref{eq:solution}. $\Box$

%% file: 5-algorithmic_details.tex
\section{Algorithmic Details}
\label{sec:algorithmic_details}

In this section, we present algorithms that transform the mathematical framework from the previous section into a practical, implementable form.  Pseudocode for our approach is found in Algorithms \ref{alg:wasp_setup}--\ref{alg:close_enough}.   


\algWASPsetup
\algGetTangentBundle
\algWASP
\algCloseEnough

\subsection{Approximate Differentiation as Iterative Process}

Our algorithm for computing approximate derivatives for a sequence of inputs ($\mathbf{x}_k$, $\mathbf{x}_{k+1}$, ...) is structured as an iterative process. This process centers around three matrices introduced in previous sections: (1) $\mathbf{D}^*$, the current approximate derivative matrix; (2) $\hat{\Delta \mathbf{F}}$, the current matrix of approximate JVPs; and (3) $\Delta \mathbf{X}$, the matrix of tangents.

Given an input in the sequence, $\mathbf{x}_k$, our algorithm aims to compute an approximate derivative at that input, $\hat{\frac{\partial f}{\partial \mathbf{x}}}\big|_{\mathbf{x}_k}$. The process begins by using the $\Delta \mathbf{X}$ matrix and current $\hat{\Delta \mathbf{F}}$ matrix to compute an updated version of $\mathbf{D}^*$. Subsequently, the new $\mathbf{D}^*$ matrix is used to update $\hat{\Delta \mathbf{F}}$. These two steps are repeated iteratively for the current input $\mathbf{x}_k$ until there is evidence that the current $\mathbf{D}^*$ matrix is close enough to the ground truth derivative matrix, $\frac{\partial f}{\partial \mathbf{x}}\big|_{\mathbf{x}_k}$.

This procedure is applied to all inputs in the sequence, interleaving updates to the $\mathbf{D}^*$ and $\hat{\Delta \mathbf{F}}$ matrices on-the-fly. Detailed steps are presented in the sections below. 

\subsection{Algorithm Explanation}
\label{sec:algorithm_flow}

Our approach begins at Algorithm \ref{alg:wasp_setup}, which outputs a \textit{cache object}.  This cache object holds key data that will be utilized during runtime, such as the tangent matrix, $\Delta \mathbf{X}$ (initialized in Algorithm  \ref{alg:tangent_bundle_matrix}) and the web of affine spaces matrix, $\hat{\Delta \mathbf{F}}$.  Algorithm  \ref{alg:wasp_setup} serves as a one-time preprocessing step, with no part of this subroutine being re-executed at runtime. 

The runtime component of our algorithm is detailed in Algorithm  \ref{alg:wasp}. This subroutine takes as input the function to be differentiated, $f$, the current input at which the derivative will be approximated, $\mathbf{x}_k$, the number of function inputs, $n$, a cache object generated by Algorithm  \ref{alg:wasp_setup}, and two distance threshold values, $d_\theta$ and $d_\ell$.  This algorithm consists of five main steps:

\vspace{4pt}

\subsubsection{Ground-truth JVP computation (Alg. \ref{alg:wasp}, lines 4--5)}

A ground truth JVP, $\Delta \mathbf{f}_i$, is computed in the direction $\Delta \mathbf{x}_i$ at the given input $\mathbf{x}_k$ using Equation \ref{eq:finite_diff}. 

\vspace{4pt}

\begin{figure*}[t!]
\centering
\includegraphics[width=\textwidth]{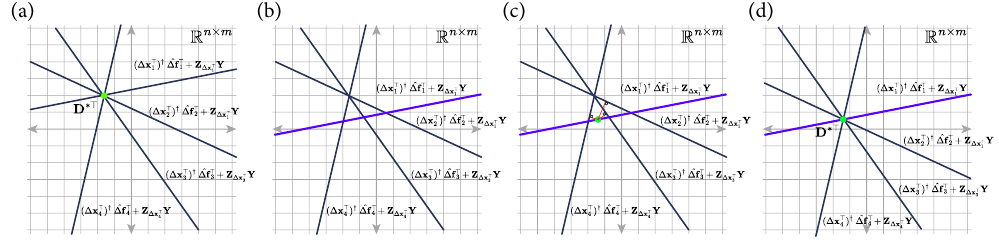}
\caption{
(a) The web of affine spaces, encoded as the columns of the $\hat{\Delta \mathbf{F}}$ matrix, start an iteration as intersecting at the previously computed derivative solution, $\mathbf{D}^{* \top}$.  (b) When a ground truth directional derivative is computed ($\Delta \mathbf{f}_1$ in this case), the affine space associated with its tangent direction ($\Delta \mathbf{x}_1$ in this case) is shifted away from the other affine spaces.  This affine space, illustrated as a purple line, is guaranteed to contain the transpose of the ground truth derivative at the current input.  (c) The constrained optimization step locates the point on the solution space that is closest (in terms of Euclidean distance) to the web of other affine spaces.  (d) This point is the transpose of the approximate derivative at the current input, and the web of affine spaces (via the $\hat{\Delta \mathbf{F}}$ matrix) is updated such that they now intersect at this new point.  The space is now ready for either another iteration of the algorithm on the same input, if needed, or the next input in the sequence, $\mathbf{x}_{k+1}$.  
}
\label{fig:explanation}
\vspace{-0.4cm}
\end{figure*}

\subsubsection{Error detection and correction (Alg. \ref{alg:wasp}, lines 6--9 \& 17--18)}


Error detection involves comparing the current approximation of the $i$-th JVP, $\hat{\Delta \mathbf{f}}_i$, with the just computed ground-truth JVP, $\Delta \mathbf{f}_i$.  These vectors are compared using Algorithm \ref{alg:close_enough}.  Specifically, this subroutine checks whether the angle and norm between the two vectors are below specified threshold values, $d_\theta$ and $d_\ell$, respectively.  If this subroutine returns \textit{True}, it indicates that $\hat{\Delta \mathbf{f}_i}$ aligns closely in direction and magnitude with $\Delta \mathbf{f}_i$, suggesting that the approximate derivative used to compute $\hat{\Delta \mathbf{f}}_i$ is likely a good approximation of the ground-truth derivative.  Conversely, if this subroutine returns \textit{False}, the current approximate derivative must not match the true derivative, and another iteration of the algorithm is taken on the same input $\mathbf{x}_k$.  This loop continues until the approximate JVP is deemed close enough to the ground truth JVP (lines 17--18).

\vspace{4pt}

\subsubsection{Ground-truth JVP update (Alg. \ref{alg:wasp}, line 10)}

Prior to line 10 in Algorithm \ref{alg:wasp}, the $\hat{\Delta \mathbf{F}}$ matrix satifies the equation $\mathbf{D}^* \Delta \mathbf{X} = \hat{\Delta \mathbf{F}}$, where $\mathbf{D}^*$ here is the most recently computed approximate derivative.  In other words, recalling \S\ref{sec:web_of_affine_spaces}, the affine spaces $(\Delta \mathbf{x}_j^\top)^\dagger \hat{\Delta \mathbf{f}_j}^\top + \mathbf{Z}_{\Delta \mathbf{x}_j^\top}\mathbf{Y}$ for all $j \in \{1, ..., n\}$ (where $\mathbf{Z}_{\Delta \mathbf{x}_j^\top}$ is the null-space matrix of the $1 \times n$ matrix $\Delta \mathbf{x}_j^\top$) only intersect at a single point: the transpose of the previously computed solution $\mathbf{D}^{* \top}$, illustrated in Figure \ref{fig:explanation}a.  

After the current approximation of the $i$-th JVP, $\hat{\Delta \mathbf{f}}_i$, is compared with the just computed ground-truth JVP, $\Delta \mathbf{f}_i$, the ground truth can now replace the approximation in the web of affine spaces matrix, $\hat{\Delta \mathbf{F}}$.  This effectively shifts the affine space associated with the $i$-th JVP, leaving the other $n-1$ affine spaces still intersecting at the previous solution $\mathbf{D}^{* \top}$.  This shift is illustrated in Figure \ref{fig:explanation}b.

\vspace{4pt} 

\subsubsection{Optimization (Alg. \ref{alg:wasp}, lines 10--14)}

Line 14 reflects the mathematical procedure specified in Equation \ref{eq:solution_d}.  This process locates the point on the affine space $(\Delta \mathbf{x}_i^\top)^\dagger \Delta \mathbf{f}_i^\top + \mathbf{Z}_{\Delta \mathbf{x}_i^\top}\mathbf{Y}$ that is closest (in terms of Euclidean distance) to the other $n-1$ affine spaces that are still intersecting at the previous solution, illustrated in Figure \ref{fig:explanation}c.  Cached matrices $\mathbf{C}_1$ and $\mathbf{C}_2$ are used to speed up this result without needing to compute matrix inverses at runtime.  The output from this step is a new matrix $\mathbf{D}^{* \top}$.  

\vspace{4pt} 

\subsubsection{Web of affine spaces matrix update (Alg. \ref{alg:wasp} line 15)}

The web of affine spaces matrix is updated such that $\mathbf{D}^* \Delta \mathbf{X} = \hat{\Delta \mathbf{F}}$.  After this update, the affine spaces within $\hat{\Delta \mathbf{F}}$ will all intersect again at the just computed $\mathbf{D}^{* \top}$, illustrated in Figure \ref{fig:explanation}d.  The matrix is now ready for either another iteration of the algorithm on the same input, if needed, or the next input in the sequence, $\mathbf{x}_{k+1}$.

\subsection{Initializing and Updating Matrices}

Two key components of our approach are the tangent matrix, $\Delta \mathbf{X}$, and the web of affine spaces matrix, $\hat{\Delta \mathbf{F}}$. The tangent matrix is initialized in Algorithm  \ref{alg:tangent_bundle_matrix}, where it is specifically constructed as a random orthonormal matrix, meaning its columns have unit length and are mutually orthogonal. For interested readers, we provide full analysis and rationale for this structure in the Appendix \S\ref{sec:tangent_matrix_structure}. To generate a random orthonormal matrix, we apply singular value decomposition (SVD) to a uniformly sampled random $n \times n$ matrix. 

The web of affine spaces matrix is initialized as a zero matrix in algorithm  \ref{alg:wasp_setup}.  This matrix will dynamically update through the error detection and correction mechanism as needed.  For instance, on the first call to Algorithm  \ref{alg:wasp}, the \texttt{close\_enough} function will almost surely return $False$ for several iterations, allowing the matrix to progressively improve in accuracy over these updates.    

\subsection{Run-time Analysis}

In this section, we analyze the run-time of our approach compared to alternatives.  We will use the notation $\texttt{rt}(.)$ to denote the run-time of a subroutine.  Approximate runtimes for several algorithms can be seen in Table \ref{tab:runtimes}.    

For WASP, $\mathbf{P}\mathbf{Q}$ and $\mathbf{R}\mathbf{S}$ are the matrix multiplications in Equation \ref{eq:solution_d} (after preprocessing) and $q$ is the number of iterations needed to achieve sufficient accuracy.  In many cases, $q=1$, though note that even in the worst case, it is guaranteed that $q \leq n$ because when $k=n$, all columns in $\hat{\Delta \mathbf{F}}$ will be ground-truth JVPs.  

Comparing to other approaches, we see that WASP shifts the computational burden of scaling $m$ and $n$ to matrix multiplications rather than repeated forward or reverse calls to $f$. This adjustment is expected to yield run-time improvements when $\texttt{rt}(f) > \texttt{rt}(\mathbf{P}\mathbf{Q}) + \texttt{rt}(\mathbf{R}\mathbf{S})$, particularly when a low value of $q$ is achievable due to a sequence of closely related inputs.

\begin{table}[t!]
    \centering
    \caption{Approximate run-times for derivative computation approaches}
    \renewcommand{\arraystretch}{1.5} 
    \begin{tabular}{|>{\centering\arraybackslash}m{0.2\linewidth}|>{\centering\arraybackslash}m{0.7\linewidth}|}
        \hline
         \textit{Approach} & \textit{Approximate runtime} \\ \hline
         \textbf{WASP} & 
         $$ 
         \begin{gathered}
         \texttt{rt}(f) + q [ \ \texttt{rt}(f) + \texttt{rt}(\underbrace{\mathbf{P}}_{\mathbb{R}^{n \times n}}\underbrace{\mathbf{Q}}_{\mathbb{R}^{n \times m}}) + \\
         \texttt{rt}(\underbrace{\mathbf{R}}_{\mathbb{R}^{n \times 1}}\underbrace{\mathbf{S}}_{\mathbb{R}^{1 \times m}}) \ ]
         \end{gathered}
         $$ \\ \hline
        \textbf{Finite Differencing} & 
        $$ 
        (n+1) \cdot \texttt{rt}(f)
        $$ \\ \hline
        \textbf{Forward AD} & 
        $$ 
        n \cdot \texttt{rt}(\texttt{with\_tangents}(f))
        $$ \\ \hline
        \textbf{Reverse AD} & 
        $$ 
        \begin{gathered}
        \texttt{rt}(\texttt{build\_computation\_graph}(f)) + \\
        m \cdot \texttt{rt}(\texttt{reverse}(f))
        \end{gathered}
        $$ \\ \hline
    \end{tabular}
    \label{tab:runtimes}
    \vspace{-10pt}
\end{table}

%% file: 6-evaluation1.tex
\section{Evaluation 1: Comparison on Benchmark Function}
\label{sec:evaluation}

In Evaluation 1, we compare our approach to several other derivative computation approaches on a benchmark function.


\subsection{Procedure}
\label{sec:procedure}

Evaluation 1 follows a three step procedure: (1) The benchmark function shown in Algorithm \ref{alg:benchmark} in initialized with given parameters $n$, $m$, and $o$.  This function will remain fixed and deterministic through the remaining steps; (2) a random walk trajectory is generated following the process seen in Algorithm \ref{alg:get_random_walk} with a given number of waypoints ($w$), dimensionality ($n$), and step length ($\lambda$); (3) For all conditions, derivatives of the benchmark function are computed in order on the $w$ inputs in the random walk trajectory.  This trajectory is kept fixed for all conditions.  Metrics are recorded for all conditions.       

The benchmark function used in this experiment is a randomly generated composition of sine and cosine functions. This function, detailed in Algorithm \ref{alg:benchmark}, was designed to be highly parameterizable, allowing for any number of inputs ($n$), outputs ($m$), and operations per output ($o$). Sine and cosine were selected for their smooth derivatives and composability, given their infinite domain and bounded range.  Moreover, numerous subroutines in robotics and related fields involve many compositions of sine and cosine functions, making this function a reasonable analogue of these processes.

Evaluation 1 is divided into several sub-experiments, detailed below. Each sub-experiment varies which parameters of the procedure are allowed to change or remain fixed, aiming to assess different facets of the differentiation process.


\algBenchmark
\algGetRandomWalk


\subsection{Conditions}
\label{sec:conditions}

Evaluation 1 compares five conditions: 

\begin{enumerate}
    \item Reverse-mode automatic differentiation with PyTorch \cite{paszke2019pytorch} backend (abbreviated as \textit{RAD-PyTorch})
    \item Finite-differencing with NumPy \cite{harris2020array} backend (abbreviated as \textit{FD})
    \item Simultaneous Perturbation Stochastic Approximation \cite{spall1992multivariate} with NumPy \cite{harris2020array} backend (abbreviated as \textit{SPSA})
    \item Web of Affine Spaces Optimization with orthonormal $\Delta \mathbf{X}$ matrix and NumPy \cite{harris2020array} backend (abbreviated as \textit{WASP-O}).
    \item Web of Affine Spaces Optimization with random, non-orthonormal $\Delta \mathbf{X}$ matrix and NumPy \cite{harris2020array} backend (abbreviated as \textit{WASP-NO}).
\end{enumerate}

All conditions in this section are implemented in Python and executed on a Desktop computer with an Intel i9 4.4GHz processor and 32 GB of RAM.  To ensure a fair comparison, the underlying benchmark function code remained consistent across all conditions, with backend switching managed via Tensorly\footnote{\href{https://tensorly.org/stable/index.html}{https://tensorly.org/stable/index.html}}.  The conditions in this section were required to remain fully compatible with the given benchmark function implemented in Tensorly, without any modifications, optimizations, or code adjustments specific to individual conditions.

We note that JAX \cite{jax2018github}, a widely-used automatic differentiation library, is not included in this section. Although JAX is theoretically compatible with Tensorly, we found that extensive code modifications were required to enable optimal performance through just-in-time (JIT) compilation. Preliminary tests showed that JAX, when not JIT-compiled, exhibited run-times that were unreasonably slow and did not reflect its full potential. Therefore, we chose not to report non-JIT-compiled JAX results in this section. 

For readers interested in further insights, supplementary results are provided in the Appendix (\S\ref{sec:evaluation1_supplement}). These results relax the Tensorly-based uniformity constraints, allowing modifications, optimizations, and compilations to strive for maximum possible performance. Conditions in this supplemental section include JAX implementations compiled for both CPU and GPU, as well as Rust-based implementations.





\subsection{Metrics}
\label{sec:metrics}

We record and report on three metrics in Evaluation 1:

\begin{enumerate}
    \item Average runtime (in seconds) of derivative computation through the sequence of $w$ inputs.
    \item Average number of calls to the \texttt{benchmark} function for a derivative computation through the sequence of $w$ inputs.  Note that this value will be constant for all conditions aside from WASP.
    \item Average accuracy of the derivative through the sequence of $w$ inputs.  We measure accuracy as the sum of \textit{angular error} (Algorithm \ref{alg:angular_error}) and \textit{norm error} (Algorithm \ref{alg:norm_error}).  The components of this error measure to what extent the rows of the returned derivative are facing the correct direction and have the correct magnitude, respectively.  If this value is zero, the returned derivative exactly matches the ground-truth derivative.
\end{enumerate}


\algAngularError
\algNormError

\subsection{Sub-experiment 1: Gradient Calculations}
\label{sec:subexp1}

In sub-experiment 1, we run the procedure outlined in \S\ref{sec:procedure} with parameters, $m=1$, $o=1000$, $w=100$, $\lambda=0.05$, $d_{\theta} = 0.1$, $d_{\ell} = 0.1$, and $n=\{ 1, 50, 100, 150, ..., 1000 \}$.  Our goal in sub-experiment 1 is to observe how the different conditions scale as the number of function inputs $n$ grows while $m$ remains fixed at $1$.  In other words, the benchmark function here is a scalar function and its derivative is a $1 \times n$ row-vector gradient.

Results for sub-experiment 1 can be seen in Figure \ref{fig:se1} (top row).  We observe that both WASP conditions outperform all other methods in runtime, except for SPSA, up to approximately $600$ inputs. Beyond this point, RAD-PyTorch. Additionally, the WASP conditions require orders of magnitude fewer function calls compared to FD, reflecting the goal of WASP to reuse recent information to avoid redundant calculation at the current input. While SPSA achieves the fastest runtime in sub-experiment 1, it incurs significant error. In contrast, the WASP conditions demonstrate much higher accuracy. Notably, the WASP condition utilizing the orthonormal tangent matrix structure maintains very low error, even for functions with up to 1000 inputs.

\begin{figure*}[t!]
\centering
\includegraphics[width=\textwidth]{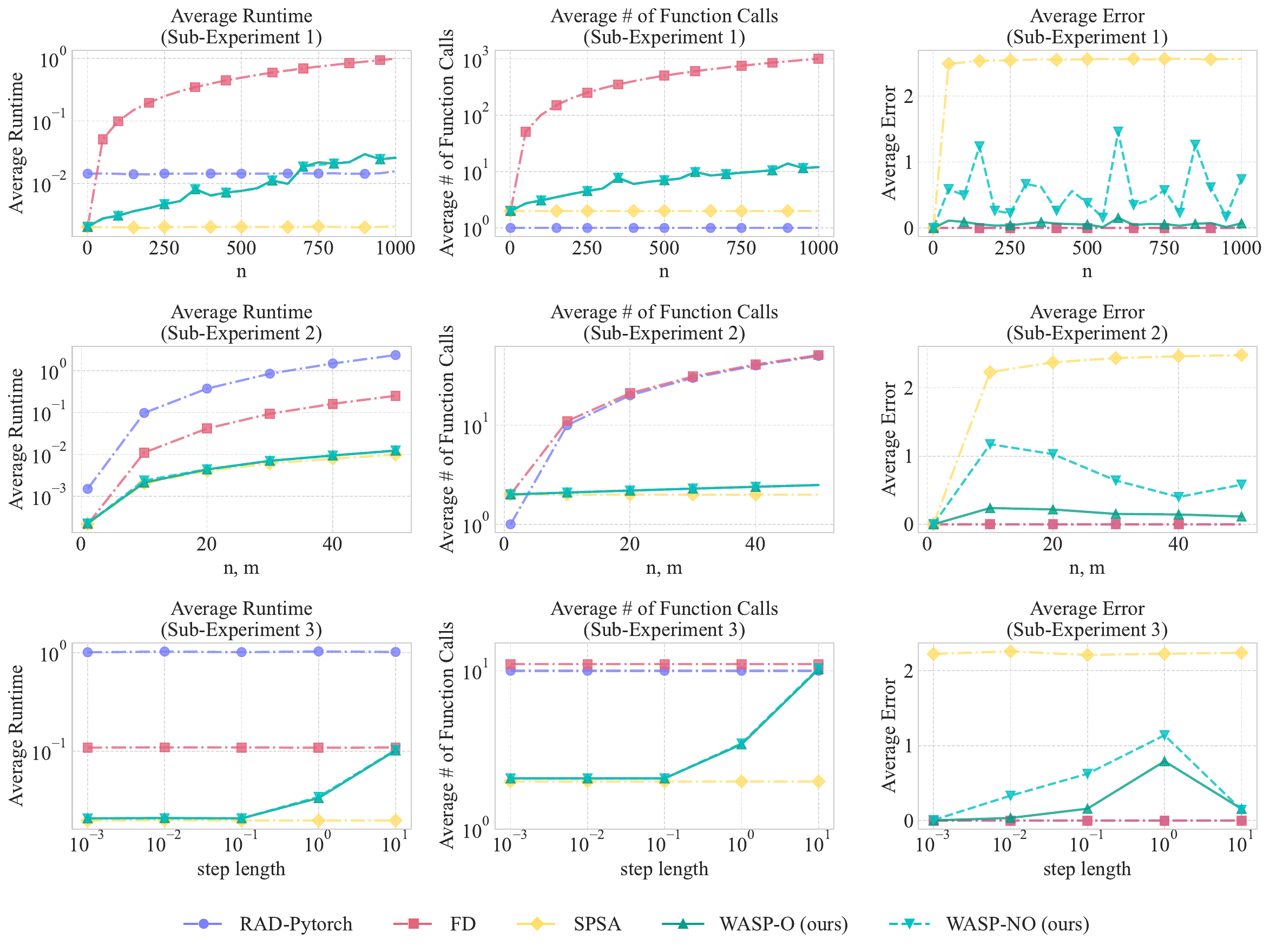}
\caption{Results for Evaluation 1, Sub-Experiment 1 (top) Sub-Experiment 2 (middle) and Sub-Experiment 3 (bottom) }
\label{fig:se1}
\vspace{-0.4cm}
\end{figure*}

\subsection{Sub-experiment 2: Square Jacobian Calculations}
\label{sec:subexp2}

In sub-experiment 2, we run the procedure outlined in \S\ref{sec:procedure} with parameters $(n,m) = (x, x)$ where $x \in \{ 1, 10, 20, 30, 40, 50 \}$, $o=1000$, $w=100$, $\lambda=0.05$, $d_{\theta} = 0.1$, and $d_{\ell} = 0.1$.  Our goal in sub-experiment 2 is to observe how the different conditions scale as the number of function inputs and number of function outputs both grow.  Thus, the benchmark function here is a vector function with the same number of inputs and outputs, and its derivative is an $n \times n$ square Jacobian.

Results for sub-experiment 2 can be seen in Figure \ref{fig:se1} (middle row).  The WASP conditions demonstrate greater efficiency compared to RAD-PyTorch and FD, achieving a runtime comparable to SPSA. This efficiency advantage is likely primarily due to the lower number of function calls, as seen in the middle graph.  Notably, WASP exhibits much lower error than SPSA, particularly in the variant incorporating the orthonormal matrix structure. This demonstrates that WASP more accurately approximates ground-truth Jacobian matrices.


\subsection{Sub-experiment 3: Varying step size}
\label{sec:subexp3}

In sub-experiment 3, we run the procedure outlined in \S\ref{sec:procedure} with parameters, $n=10$, $m=10$, $o=1000$, $w=100$, $d_{\theta} = 0.1$, $d_{\ell} = 0.1$, and $\lambda=\{ 0.001, 0.01, 0.1, 1, 10 \}$.  Our goal in sub-experiment 3 is to observe how the different conditions scale as the step-length in the random walk trajectory grows.  


Results for sub-experiment 3 can be seen in Figure \ref{fig:se1} (bottom row).  As expected, the performance of the WASP conditions generally declines as the step length increases. However, interestingly for a step length of $\lambda = 10$, the results show that the error approaches nearly zero, seemingly outperforming the trials with $\lambda = 1$ in terms of accuracy. This improved accuracy, however, comes at the cost of significantly more function calls, resulting in a much higher average runtime. Essentially, a step length of $\lambda = 10$ in this case was so large that the error detection and correction mechanism consistently triggered additional iterations until reaching the upper limit ($n$). As a result, the WASP conditions effectively defaulted to a standard finite-differencing strategy, hence why the WASP conditions are equivalent to the FD condition in terms of runtime and number of function calls in this scenario.  


%% file: 7-evaluation2.tex
\section{Evaluation 2: Error Propagation Analysis}

In \S\ref{sec:algorithm_flow}, we described the error detection and correction mechanism in our algorithm, which is designed to prevent error accumulation over a sequence of approximate derivative computations. In Evaluation 2, we analyze how errors propagate through our algorithm under different parameter settings, thereby evaluating the effectiveness of our proposed mechanism over long derivative sequences.

\begin{figure*}[t!]
\centering
\includegraphics[width=\textwidth]{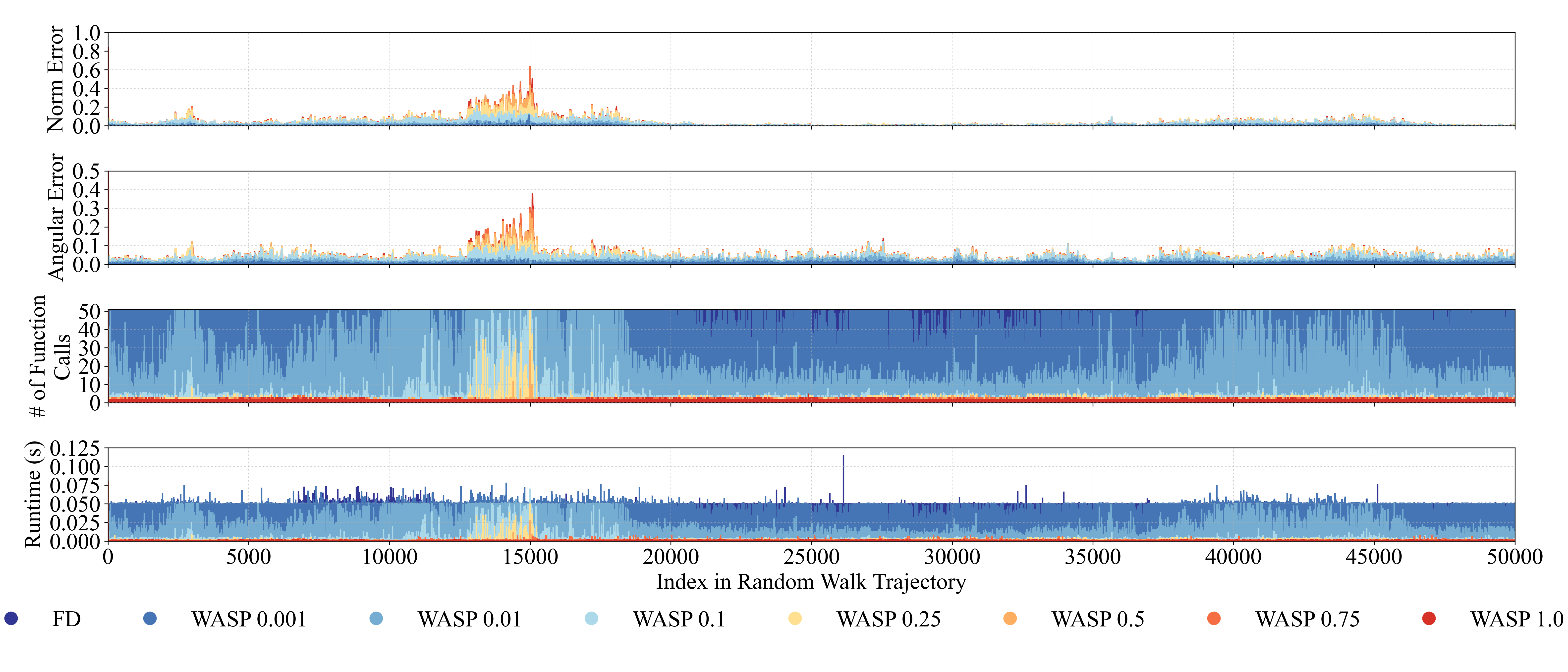}
\caption{Results for Evaluation 2.  These results show the norm error (first row), angular error (second row), the number of function calls (third row), and runtime (fourth row) per derivative computation over a sequence of 50{,}000 derivatives (x-axis). }
\label{fig:se4}
\vspace{-0.3cm}
\end{figure*}

\subsection{Procedure}

Evaluation 2 follows the same procedure as Evaluation 1, described in \S\ref{sec:procedure}.  We use parameters $n = 50$, $m = 1$, $o = 1000$, $w = 50{,}000$, $\lambda = 0.05$, where $n$ is the number of inputs to the benchmark function, $m$ is the number of outputs from the benchmark function, $o$ is the number of operations per output in the benchmark function, $w$ is the number of waypoints in the random walk trajectory, and $\lambda$ is the step length along the random walk trajectory.  

\subsection{Conditions}

The primary values we are varying and assessing in Evaluation 2 are the error threshold parameters, $d_\theta$ and $d_\ell$, outlined in \S\ref{sec:algorithm_flow}. 
 Specifically, we use parameter settings $(d_{\theta}, d_{\ell}) = (x, x)$ where $x \in \{ 0.001, 0.01, 0.1, 0.25, 0.5, 0.75, 1.0 \}$. These settings allow us to evaluate error behavior across different error thresholds (the $d_{\theta}$ and $d_{\ell}$ parameters) over a long input sequence (as specified by the $w$ parameter above). 
 
 The WASP method in this evaluation uses a fixed orthonormal $\Delta \mathbf{X}$ matrix shared across all $d_{\theta}$ and $d_{\ell}$ configurations. We also compare the WASP variants against standard Finite Differencing using a NumPy backend.  
 
 All conditions in Evaluation 2 are implemented in Python using Tensorly for backend switching and executed on a Desktop computer with an Intel i7 5.4GHz processor and 32 GB of RAM.

\subsection{Metrics}

We record and report on four metrics in Evaluation 2:

\begin{enumerate}
    \item Runtime (in seconds) per each derivative computation through the sequence of $w$ inputs.
    \item Number of calls to the \texttt{benchmark} function for each derivative computation through the sequence of $w$ inputs.
    \item The angular error of each derivative through the sequence of $w$ inputs (Algorithm \ref{alg:angular_error}).  
    \item The norm error of each derivative through the sequence of $w$ inputs (Algorithm \ref{alg:norm_error}).
\end{enumerate}

\subsection{Results}

Results for Eavluation 2 are shown in Figure \ref{fig:se4}. At a high level, we observe that WASP does not accumulate significant error over long sequences, even when high error thresholds are used. For instance, even at the $50{,}000$-th input, the errors remain low. In general, there are subtle ebbs and flows in error, naturally requiring more or fewer function calls throughout the sequence. At certain points, such as between the $13{,}000$-th and $15{,}000$-th inputs, WASP exhibits a transient increase in error under high thresholds, suggesting that this portion of the input sequence is less compatible with the WASP heuristic. However, even in these cases, the error remains within reasonable and usable bounds (e.g., less than $0.4$ radians from the ground truth gradient), and the algorithm successfully self-corrects after these brief periods of elevated error without diverging. As expected, using lower error thresholds consistently results in low error, but at the cost of additional function calls and runtime. In the limit as $d_{\theta}$ and $d_{\ell}$ approach $0$, both the accuracy and runtime performance converge to that of full finite differencing.

%% file: 8-evaluation3.tex
\section{Evaluation 3: Application in Robot Optimization}

In Evaluation 3, we compare our approach to several other derivative computation approaches in a robotics-based root-finding procedure.

\begin{figure}[t!]
\centering
\includegraphics[width=\columnwidth]{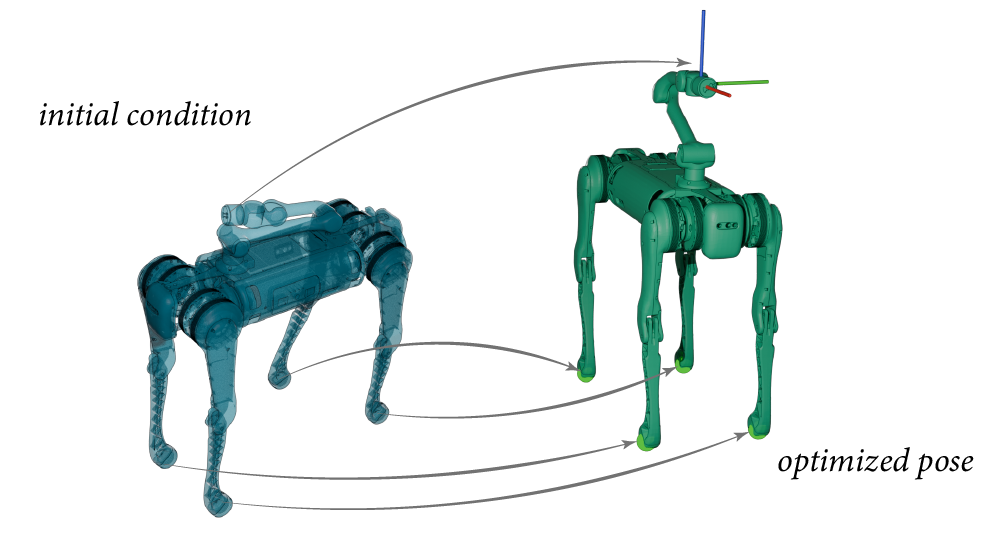}
\caption{ Evaluation 3 involves assessing performance in a robotic root-finding procedure, where the goal is to determine a robot pose that positions its feet and end-effector at predefined locations or orientations. }
\label{fig:robots}
\vspace{-0.2cm}
\end{figure}

\subsection{Procedure}

Evaluation 3 follows a three step procedure: (1) A robot state is sampled for a simulated Unitree B1 quadruped robot\footnote{\href{https://www.unitree.com/b1}{https://www.unitree.com/b1}} with a Z1 manipulator\footnote{\href{https://www.unitree.com/z1}{https://www.unitree.com/z1}} mounted on its back (shown in Figure \ref{fig:robots}).  This robot has 24 degrees of freedom (including a floating base to account for mobility).  This sampled state, $\mathbf{x}_0 \in \mathbb{R}^{24}$, will be an initial condition for an optimization process; (2) A Jacobian pseudo-inverse method (Algorithm \ref{alg:root_finding}) is used to find a root for a constraint function (Algorithm \ref{alg:constraint_function}).  The constraint function has five outputs: four for specifying foot placements and one for specifying the end-effector pose for the manipulator mounted on the back.  Thus, the Jacobian of this constraint function is a $5 \times 24$ matrix; (3) Step 1--2 are run 50 times per condition.  Metrics are recorded for all conditions.

\algJacobianPseudoinverse
\algConstraintFunction

\subsection{Conditions}
The conditions in Evaluation 3 are the same as those listed in \S\ref{sec:conditions}.  All conditions are implemented in Python and executed on a Desktop computer with an Intel i7 5.4GHz processor and 32 GB of RAM.  Only the CPU was used for these experiments and Tensorly was again used for all backend-switching to maintain uniformity.  

\subsection{Metrics}

We record and report on two metrics in Evaluation 3:

\begin{enumerate}
    \item Average runtime (in seconds) to converge on a robot configuration sufficiently close to the constraint surface.
    \item Average number of optimization steps needed to converge on a robot configuration sufficiently close to the constraint surface.
\end{enumerate}

\subsection{Results}

Results for Evaluation 3 can be seen in Table \ref{tab:evaluation2}.  The results show that the WASP conditions achieve significantly faster convergence compared to alternative approaches. Additionally, the orthonormal structure of the tangent matrix further enhances convergence efficiency. In contrast, the SPSA condition failed to converge, highlighting that certain derivative approximation methods may lack the accuracy required for some optimization procedures. 

\begin{table}[t]
    \centering
    \caption{Evaluation 3 results.  The $^*$ symbol means that the condition never converged in the maximum number of iterations (10,000).  Range values denote standard deviation.}
    \begin{tabular}{|>{\centering\arraybackslash}p{0.27\linewidth}|>{\centering\arraybackslash}p{0.27\linewidth}|>{\centering\arraybackslash}p{0.27\linewidth}|}
        \hline
         \textit{Approach} & \textit{Average runtime (seconds)} & \textit{Average \# of iterations} \\ \hline
        \textbf{RAD-PyTorch} & 72.3 $\pm$ 3.5 & 504 $\pm$ 23 \\ \hline
        \textbf{FD} & 50.0 $\pm$ 2.5 & 504 $\pm$  23\\ \hline
        \textbf{SPSA} & 114.5$^{*}$ $\pm$ 1.6 & 10,000$^{*}$ $\pm$ 0 \\ \hline
        \textbf{WASP-O} (ours) & 7.0 $\pm$ 0.4& 544 $\pm$ 26 \\ \hline
        \textbf{WASP-NO} (ours) & 9.5 $\pm$ 3.5 &  770 $\pm$ 310\\ \hline
    \end{tabular}
    \label{tab:evaluation2}
    \vspace{-5pt}
\end{table}

%% file: 9-discussion.tex
\section{Discussion}
\label{sec:discussion}

In this work, we introduced a coherence-based approach for efficiently calculating a sequence of derivatives. Our approach leverages a novel process, the Web of Affine Spaces (WASP) Optimization, to identify a point guaranteed to reside on an affine space containing the ground-truth derivative that also aligns well with prior related calculations. Through extensive evaluations, we demonstrate that our approach outperforms widely used techniques such as automatic differentiation and finite differencing in terms of efficiency, while delivering greater accuracy compared to other derivative approximation methods. In this section, we discuss the limitations and broader implications of our approach.

\subsection{Limitations}

We note several limitations of our work that suggest future avenues of research and extensions.  First, the derivatives produced by our approach are approximations. Although our method can, in theory, achieve high accuracy comparable to standard finite differencing, attaining such precision would necessitate additional iterations, likely compromising its efficiency. For applications demanding exact derivatives in all cases, alternative techniques may be more appropriate. 


Next, this paper focuses on presenting the math, algorithms, and initial proofs of concept for the WASP derivative approach. While we believe this technique has potential for impact across robotics and other fields, exploring its full range of applications and establishing best practices across many different problems is beyond the scope of this paper.  


Furthermore, as demonstrated in Evaluation 1, the effectiveness of our approach diminishes significantly as the gap between inputs increases. We aim to refine or reformulate aspects of our approach going forward to reduce its strong sensitivity to step size. 


Additionally, while our error detection and correction technique performs well in practice, it does not guarantee the accuracy of the approximations relative to ground-truth derivatives. Currently, the approach relies on JVPs (directional derivatives) as a proxy for true derivatives. While this provides a necessary condition for correctness, it is not a sufficient condition. In other words, the true derivative will always yield matching directional derivatives, but matching directional derivatives do not necessarily guarantee the underlying derivatives are correct.  Ideally, the error detection and correction mechanism, along with its associated parameters, would relate directly to the true derivatives. However, addressing this issue is inherently challenging and perhaps infeasible.  To illustrate, such a solution seems to elicit a circular reasoning problem: if one could sufficiently estimate the ground-truth derivative well enough to bound the approximation, that same information could be used to directly select a closer approximation to the true derivative in the first place.  This issue remains an open question that requires further investigation.

Lastly, while the current approach performs well for small to medium-sized problems, it does not scale effectively to large-scale functions. For example, applying this method to neural network training with millions or billions of parameters would be infeasible due to the prohibitive size of the matrices required for the optimization process. In future work, we plan to explore scalable adaptations of this approach that better manage storage and computational demands. 


\subsection{Implications}

Due to the ubiquity of derivative computation in robotics and beyond, we believe our work has the potential for broad impact and applicability.  For example, it could prove to be valuable in areas such as model predictive control, physics simulation, trajectory optimization, inverse kinematics, and more.  Our goal is to enable the community to leverage these derivatives to streamline computationally expensive subroutines, achieving performance gains with minimal code modifications. By doing so, we aim to unlock new levels of interactivity and adaptability for robots, empowering them to operate seamlessly and responsively in real-time environments. 





%% file: 10-appendix.tex

\section{Appendix}
\label{sec:appendix}

\subsection{Geometric Interpretation of Solution}
\label{sec:geometric_interpretation_of_solution}

Here, we explain the geometric interpretation of the solution in Equation \ref{eq:solution_d}.  First, we split the solution into two terms:  

\begin{equation}
\footnotesize
\label{eq:solution_terms}
\begin{gathered}
\mathbf{D}^{\top *} = \overbrace{\mathbf{A}^{-1}(\mathbf{I}_{n \times n} - s_i^{-1}\Delta\mathbf{x}_i \Delta \mathbf{x}_i^\top \mathbf{A}^{-1})2\Delta\mathbf{X} \hat{\Delta\mathbf{F}}^\top}^{\text{Term 1}} +  \\ \underbrace{s_i^{-1}\mathbf{A}^{-1}\Delta \mathbf{x}_i\Delta \mathbf{f}_i^\top}_{\text{Term 2}}  
\\
\mathbf{A} = 2\Delta\mathbf{X} \Delta\mathbf{X}^\top, \ s_i = \Delta \mathbf{x}_i^{\top}\mathbf{A}^{-1}\Delta \mathbf{x}_i.
\end{gathered}
\end{equation}

\noindent We overview each term separately.

\vspace{15pt}

\subsubsection{Geometric Interpretation of Term 1}

We begin by looking at the first component of Term 1:

$$
\mathbf{A}^{-1}(\mathbf{I}_{n \times n} - s_i^{-1}\Delta\mathbf{x}_i \Delta \mathbf{x}_i^\top \mathbf{A}^{-1}).
$$

\noindent This matrix can be rewritten in the following form:

\begin{equation}
\begin{gathered}
\mathbf{M}_i = \mathbf{A}^{-1}\mathbf{P}_i, \\
\mathbf{P}_i = \mathbf{I}_{n \times n} - \frac{\Delta\mathbf{x}_i\Delta\mathbf{x}_i^\top\mathbf{A}^{-1}}{\Delta \mathbf{x}_i^\top \mathbf{A}^{-1}\Delta \mathbf{x}_i}. 
\end{gathered} 
\end{equation}

\noindent Below, we present several propositions that contribute to a geometric understanding.

\begin{proposition}
\label{prop:orthogonal_p}
$\mathbf{P}_i\mathbf{y} \in (\mathbf{A}^{-1}\Delta\mathbf{x}_i)^\perp \ \forall \mathbf{y} \in \mathbb{R}^n$, i.e. the matrix-vector product $\mathbf{P}_i\mathbf{y}$ is orthogonal to $\mathbf{A}^{-1} \Delta \mathbf{x}_i$ for all $\mathbf{y} \in \mathbb{R}^n$
\end{proposition}

\begin{proof}
$$
(\mathbf{A}^{-1}\Delta \mathbf{x}_i)^\top \mathbf{P}_i\mathbf{y} = \Delta\mathbf{x}_i^\top (\mathbf{A}^{-1})^\top  \mathbf{P}_i\mathbf{y}
$$

Note that $\mathbf{A}^{-1}$ is symmetric, thus $(\mathbf{A}^{-1})^\top = \mathbf{A}^{-1}$.  Multiplying out this expression:

$$
\Delta\mathbf{x}_i^\top \mathbf{A}^{-1}  \mathbf{P}_i\mathbf{y} = \Delta\mathbf{x}_i^\top \mathbf{A}^{-1}  (\mathbf{I}_{n \times n} - \frac{\Delta\mathbf{x}_i\Delta\mathbf{x}_i^\top\mathbf{A}^{-1}}{\Delta \mathbf{x}_i^\top \mathbf{A}^{-1}\Delta \mathbf{x}_i} ) \mathbf{y}
$$

$$
=  \Delta\mathbf{x}_i^\top \mathbf{A}^{-1}\mathbf{y} - \Delta\mathbf{x}_i^\top \mathbf{A}^{-1}\Delta\mathbf{x}_i\frac{\Delta\mathbf{x}_i^\top\mathbf{A}^{-1}\mathbf{y}}{\Delta \mathbf{x}_i^\top \mathbf{A}^{-1}\Delta \mathbf{x}_i} 
$$

$$
=  \Delta\mathbf{x}_i^\top \mathbf{A}^{-1}\mathbf{y} - \Delta\mathbf{x}_i^\top \mathbf{A}^{-1}\mathbf{y} = 0
$$

\end{proof}

\begin{proposition}
\label{prop:orthogonal_m}
$\mathbf{M}_i\mathbf{y} \in \Delta\mathbf{x}_i^\perp \ \forall \mathbf{y} \in \mathbb{R}^n$, i.e. the matrix-vector product $\mathbf{M}_i\mathbf{y}$ is orthogonal to $\Delta \mathbf{x}_i$ for all $\mathbf{y} \in \mathbb{R}^n$.
\end{proposition}

\begin{proof}
$$
\Delta \mathbf{x}_i^\top \mathbf{M}_i\mathbf{y} = \Delta \mathbf{x}_i^\top\mathbf{A}^{-1}\mathbf{y} - \Delta \mathbf{x}_i^\top\frac{\mathbf{A}^{-1}\Delta\mathbf{x}_i\Delta\mathbf{x}_i^\top\mathbf{A}^{-1}}{\Delta \mathbf{x}_i^\top \mathbf{A}^{-1}\Delta \mathbf{x}_i}\mathbf{y} 
$$

$$
= \Delta \mathbf{x}_i^\top\mathbf{A}^{-1}\mathbf{y} - \Delta \mathbf{x}_i^\top\mathbf{A}^{-1}\Delta\mathbf{x}_i\frac{\Delta\mathbf{x}_i^\top\mathbf{A}^{-1}\mathbf{y} }{\Delta \mathbf{x}_i^\top \mathbf{A}^{-1}\Delta \mathbf{x}_i}
$$

$$
= \Delta \mathbf{x}_i^\top\mathbf{A}^{-1}\mathbf{y} - \Delta \mathbf{x}_i^\top\mathbf{A}^{-1}\mathbf{y} = 0
$$
\end{proof}

\begin{proposition}
\label{prop:null_space_p}
$\Delta \mathbf{x}_i$ is in the null space of $\mathbf{P}_i$   
\end{proposition}

\begin{proof}
$$
\mathbf{P}_i \Delta \mathbf{x}_i = (\mathbf{I}_{n \times n} - \frac{\Delta\mathbf{x}_i\Delta\mathbf{x}_i^\top\mathbf{A}^{-1}}{\Delta \mathbf{x}_i^\top \mathbf{A}^{-1}\Delta \mathbf{x}_i}) \Delta \mathbf{x}_i
$$

$$
= \Delta \mathbf{x}_i - \Delta\mathbf{x}_i\frac{\Delta\mathbf{x}_i^\top\mathbf{A}^{-1}\Delta \mathbf{x}_i}{\Delta \mathbf{x}_i^\top \mathbf{A}^{-1}\Delta \mathbf{x}_i}
$$

$$
= \Delta \mathbf{x}_i - \Delta\mathbf{x}_i = 0
$$
\end{proof}

\begin{proposition}
\label{prop:null_space_m}
$\Delta \mathbf{x}_i$ is in the null space of $\mathbf{M}_i$
\end{proposition}

\begin{proof}
$$
\mathbf{M}_i\Delta \mathbf{x}_i = \mathbf{A}^{-1}\Delta \mathbf{x}_i - \frac{\mathbf{A}^{-1}\Delta\mathbf{x}_i\Delta\mathbf{x}_i^\top\mathbf{A}^{-1}}{\Delta \mathbf{x}_i^\top \mathbf{A}^{-1}\Delta \mathbf{x}_i} \Delta \mathbf{x}_i 
$$

$$
= \mathbf{A}^{-1}\Delta \mathbf{x}_i - \mathbf{A}^{-1}\Delta\mathbf{x}_i\frac{\Delta\mathbf{x}_i^\top\mathbf{A}^{-1}\Delta \mathbf{x}_i}{\Delta \mathbf{x}_i^\top \mathbf{A}^{-1}\Delta \mathbf{x}_i}  
$$

$$
= \mathbf{A}^{-1}\Delta \mathbf{x}_i - \mathbf{A}^{-1}\Delta\mathbf{x}_i = \mathbf{0}
$$
\end{proof}

\begin{proposition}
\label{prop:null_space_span_p}
$\Delta \mathbf{x}_i$ spans the whole null space of $\mathbf{P}_i$
\end{proposition}

\begin{proof}
We already know $\Delta \mathbf{x}_i$ is in the null space of $\mathbf{P}_i$ from Proposition \ref{prop:null_space_p}.  Now, suppose $\mathbf{z}$ is a vector in the null space of $\mathbf{P}_i$ that is not a scaling of $\Delta \mathbf{x}_i$, i.e., $\mathbf{z} \neq \lambda \Delta \mathbf{x}_i$.  In this case, we would have the following:

$$
\mathbf{P}_i\mathbf{z} = \mathbf{z} - \frac{\Delta\mathbf{x}_i\Delta\mathbf{x}_i^\top\mathbf{A}^{-1}}{\Delta \mathbf{x}_i^\top \mathbf{A}^{-1}\Delta \mathbf{x}_i}\mathbf{z} = \mathbf{0}
$$

$$
\mathbf{z} = \Delta\mathbf{x}_i\frac{\Delta\mathbf{x}_i^\top\mathbf{A}^{-1}\mathbf{z}}{\Delta \mathbf{x}_i^\top \mathbf{A}^{-1}\Delta \mathbf{x}_i}
$$

$$
\mathbf{z} = \lambda \Delta\mathbf{x}_i
$$

\noindent where $\lambda = \frac{\Delta\mathbf{x}_i^\top\mathbf{A}^{-1}\mathbf{z}}{\Delta \mathbf{x}_i^\top \mathbf{A}^{-1}\Delta \mathbf{x}_i} \in \mathbb{R}$.  This contradicts our assumption that $\mathbf{z}$ is not a scaling of $\Delta \mathbf{x}_i$, meaning $\Delta \mathbf{x}_i$ must span the whole null space of $\mathbf{P}_i$.

\end{proof}

\begin{proposition}
\label{prop:null_space_span_m}
$\Delta \mathbf{x}_i$ spans the whole null space of $\mathbf{M}_i$
\end{proposition}

\begin{proof}
We already know $\Delta \mathbf{x}_i$ is in the null space of $\mathbf{M}_i$ from Proposition \ref{prop:null_space_m}.  Now, suppose $\mathbf{z}$ is a vector in the null space of $\mathbf{M}_i$ that is not a scaling of $\Delta \mathbf{x}_i$, i.e., $\mathbf{z} \neq \lambda \Delta \mathbf{x}_i$.  In this case, we would have the following:

$$
\mathbf{M}_i\mathbf{z} = \mathbf{A}^{-1}\mathbf{z} - \frac{\mathbf{A}^{-1}\Delta\mathbf{x}_i\Delta\mathbf{x}_i^\top\mathbf{A}^{-1}}{\Delta \mathbf{x}_i^\top \mathbf{A}^{-1}\Delta \mathbf{x}_i}\mathbf{z} = \mathbf{0}
$$

$$
\mathbf{A}^{-1}(\mathbf{z} - \frac{\Delta\mathbf{x}_i\Delta\mathbf{x}_i^\top\mathbf{A}^{-1}\mathbf{z}}{\Delta \mathbf{x}_i^\top \mathbf{A}^{-1}\Delta \mathbf{x}_i}) = \mathbf{0}.
$$

\noindent We know that $\mathbf{A}^{-1}$ must be invertible by definition, implying that the term in parentheses must equal zero:

$$
\mathbf{z} - \frac{\Delta\mathbf{x}_i\Delta\mathbf{x}_i^\top\mathbf{A}^{-1}\mathbf{z}}{\Delta \mathbf{x}_i^\top \mathbf{A}^{-1}\Delta \mathbf{x}_i} = \mathbf{0}
$$

$$
\mathbf{z} = \Delta\mathbf{x}_i\frac{\Delta\mathbf{x}_i^\top\mathbf{A}^{-1}\mathbf{z}}{\Delta \mathbf{x}_i^\top \mathbf{A}^{-1}\Delta \mathbf{x}_i}
$$

$$
\mathbf{z} = \lambda \Delta\mathbf{x}_i
$$

\noindent where $\lambda = \frac{\Delta\mathbf{x}_i^\top\mathbf{A}^{-1}\mathbf{z}}{\Delta \mathbf{x}_i^\top \mathbf{A}^{-1}\Delta \mathbf{x}_i} \in \mathbb{R}$.  This statement contradicts our assumption that $\mathbf{z}$ is not a scaling of $\Delta \mathbf{x}_i$, meaning $\Delta \mathbf{x}_i$ must span the whole null space of $\mathbf{M}_i$.

\end{proof}

\begin{proposition}
\label{prop:rank_p}
$rank(\mathbf{P}_i) = n - 1$
\end{proposition}

\begin{proof}
The rank of $\mathbf{P}_i$ must satisfy $n = rank(\mathbf{M}_i) + nullity(\mathbf{P}_i)$.  We know the nullity of $\mathbf{P}_i$ is $1$ from Proposition \ref{prop:null_space_span_p}, meaning we have $n = rank(\mathbf{P}_i) + 1 \Longrightarrow rank(\mathbf{P}_i) = n - 1$.
\end{proof}

\begin{proposition}
\label{prop:rank_m}
$rank(\mathbf{M}_i) = n - 1$
\end{proposition}

\begin{proof}
The rank of $\mathbf{M}_i$ must satisfy $n = rank(\mathbf{M}_i) + nullity(\mathbf{M}_i)$.  We know the nullity of $\mathbf{M}_i$ is $1$ from Proposition \ref{prop:null_space_span_m}, meaning we have $n = rank(\mathbf{M}_i) + 1 \Longrightarrow rank(\mathbf{M}_i) = n - 1$.
\end{proof}

\begin{figure*}[t!]
\centering
\includegraphics[width=\textwidth]{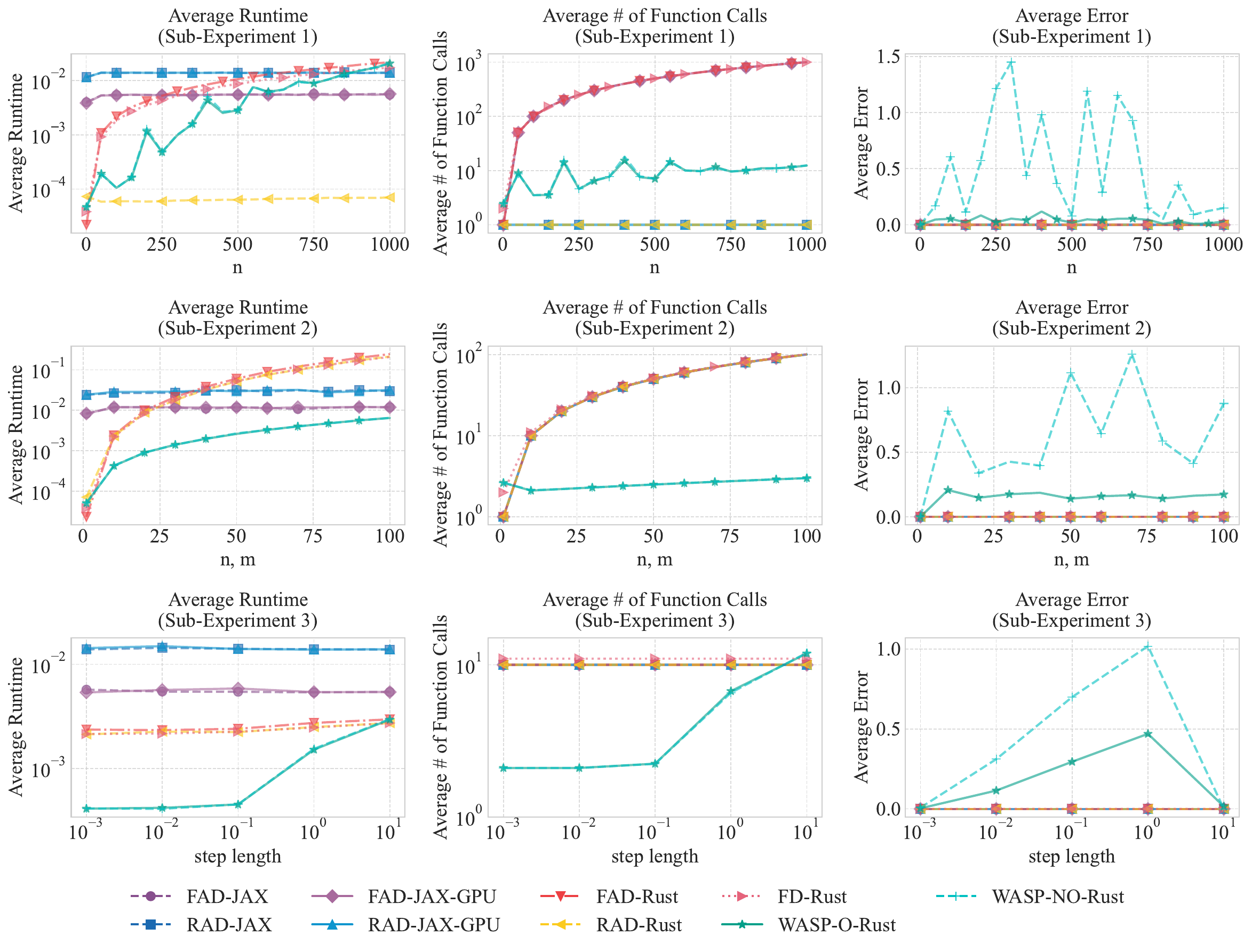}
\caption{Results for Evaluation 1--supplemental conditions: Sub-experiment 1 (top) Sub-experiment 2 (middle) and Sub-experiment 3 (bottom) }
\label{fig:se1_supplement}
\end{figure*}

\vspace{10pt}

In summary, $\mathbf{P}_i$ projects any $n$-vector onto the subspace of $\mathbb{R}^n$ that is orthogonal to $\mathbf{A}^{-1} \Delta \mathbf{x}_i$ (Propositions \ref{prop:orthogonal_p}, \ref{prop:null_space_span_p}, and  \ref{prop:rank_p}), and $\mathbf{M}_i$ projects any $n$-vector onto the subspace of $\mathbb{R}^n$ that is orthogonal to $\Delta \mathbf{x}_i$ (Propositions \ref{prop:orthogonal_m}, \ref{prop:null_space_span_m}, and  \ref{prop:rank_m}).  Concretely, let us examine the effect of mapping a vector $\mathbf{y}$ using the matrix $\mathbf{M}_i$:

$$
\mathbf{M}_i\mathbf{y} = \mathbf{A}^{-1}\mathbf{P}_i\mathbf{y}.
$$

\noindent Here, we see that $\mathbf{y}$ is first transformed by $\mathbf{P}_i$ into an intermediate space that is orthogonal to $\mathbf{A}^{-1}\Delta \mathbf{x}_i$.  Then, this intermediate vector is further transformed by $\mathbf{A}^{-1}$ into a space that is orthogonal to $\Delta \mathbf{x}_i$.  



Note that the $\mathbf{M}_i$ matrix functions similarly to the $\mathbf{Z}_{\Delta \mathbf{x}^\top}$ null space basis matrix in Equation \ref{eq:affine_solution_space}.  The columns of both of these matrices span a space that is orthogonal to $\Delta \mathbf{x}_i$.  Also, the matrix in the second part of Term 1, $2\Delta\mathbf{X} \hat{\Delta\mathbf{F}}^\top$, serves as an analogue to the $\mathbf{Y}$ matrix in Equation \ref{eq:affine_solution_space}. The columns of $2\Delta\mathbf{X}  \hat{\Delta\mathbf{F}}^\top$ are projected by $\mathbf{M}_i$ onto the subspace of $\mathbb{R}^n$ that is orthogonal to $\Delta \mathbf{x}_i$, identifying the region within this infinite space where an approximation of the derivative—excluding components in the $\Delta \mathbf{x}_i$ direction—resides.


\vspace{15pt}

\subsubsection{Geometric Interpretation of Term 2}
\label{sec:geometric_interpretation_of_term_2}
 
Rewriting Term 2 here for convenience:

$$
s_i^{-1}\mathbf{A}^{-1}\Delta \mathbf{x}_i \Delta\mathbf{f}_i^\top.
$$

\noindent We see that this term is a rank-1 matrix induced by an outer product.  Note that this matrix is reminiscent of the minimum norm solution $(\Delta \mathbf{x}^\top)^\dagger \Delta \mathbf{f}^\top$ matrix in Equation \ref{eq:affine_solution_space}.  This comparison with Equation \ref{eq:affine_solution_space} seemingly integrates nicely with our analysis regarding Term 1 above.  Since Term 1 provides an approximation of the derivative with all contribution from the direction $\Delta \mathbf{x}_i$ removed (analogous to $\mathbf{Z}_{\Delta \mathbf{x}^\top}\mathbf{Y}$ in Equation \ref{eq:affine_solution_space}), our expectation would be that Term 2 would restore the part of the solution only along this missing $\Delta \mathbf{x}_i$ direction (analogous to $(\Delta \mathbf{x}^\top)^\dagger \Delta \mathbf{f}^\top$ in Equation \ref{eq:affine_solution_space}).

Although the hypothetical analysis above suggests a strong geometric symmetry between Term 2 and $(\Delta \mathbf{x}^\top)^\dagger \Delta \mathbf{f}^\top$, there is a subtle complication that prevents a perfect correspondence.  While $(\Delta \mathbf{x}^\top)^\dagger \Delta \mathbf{f}^\top$ scales in the $\Delta \mathbf{x}$ direction, contributing only in this direction missing from $\mathbf{Z}_{\Delta \mathbf{x}^\top}\mathbf{Y}$, the Term 2 expression $s_i^{-1}\mathbf{A}^{-1}\Delta \mathbf{x}_i \Delta\mathbf{f}_i^\top$ scales in the $\mathbf{A}^{-1}\Delta \mathbf{x}_i$ direction—not just $\Delta \mathbf{x}_i$. As a result, this term may reintroduce parts of the solution outside of the $\Delta \mathbf{x}_i$ direction, potentially interfering with the partial solution already provided by Term 1.  We assess where this misalignment comes from and how to address this issue below. 



\subsection{Tangent Matrix Structure}
\label{sec:tangent_matrix_structure}

A key element in our approach is the matrix of tangent vectors, $\Delta \mathbf{X}$. While, in theory, $\Delta \mathbf{X}$ can be any full-rank $n \times r$ matrix—Equation \ref{eq:solution_d} can indeed yield an optimal solution with respect to any chosen $\Delta \mathbf{X}$—in this section, we explore whether enforcing a particular structure on $\Delta \mathbf{X}$ might enhance the results overall.


Our goal is to address the potential misalignment between Term 1 and Term 2 presented in \S\ref{sec:geometric_interpretation_of_term_2}.  In other words, we want Term 2 to only reintroduce parts of the solution in the $\Delta \mathbf{x}_i$ direction without interfering with the partial solution already provided by Term 1.  Note that to achieve this result, all $\Delta \mathbf{x}_i$ vectors cannot be rotated at all by their corresponding $\mathbf{A}^{-1}$ matrix.  Instead, the $\Delta \mathbf{x}_i$ must only be scaled by $\mathbf{A}^{-1}$, meaning that it must be an Eignevector of this matrix, i.e.,

$$
\mathbf{A}^{-1}\Delta\mathbf{x}_i = \lambda \Delta \mathbf{x}_i.
$$

\noindent for some scalar value $\lambda$.  Through the proposition below, we demonstrate that this result is always achieved when $\Delta \mathbf{X}$ is orthonormal.

\begin{proposition}
    When $\Delta \mathbf{X}$ is orthonormal, $\Delta \mathbf{x}_i$ is always an Eigenvector of $\mathbf{A}^{-1}$ for all $i$, meaning Term 2, $s_i^{-1}\mathbf{A}^{-1}\Delta \mathbf{x}_i\Delta \mathbf{f}_i^\top$, is always orthogonal to Term 1, $\mathbf{A}^{-1}(\mathbf{I}_{n \times n} - s_i^{-1}\Delta\mathbf{x}_i \Delta \mathbf{x}_i^\top \mathbf{A}^{-1})$
\end{proposition}

\begin{proof}
To begin, let us revisit Term 1, as discussed in \S\ref{sec:geometric_interpretation_of_solution}:

$$
\mathbf{A}^{-1}(\mathbf{I}_{n \times n} - s_i^{-1}\Delta\mathbf{x}_i \Delta \mathbf{x}_i^\top \mathbf{A}^{-1}).
$$

\noindent Here, $\mathbf{A} = 2\Delta\mathbf{X}\Delta\mathbf{X}^\top$ and $s_i = \Delta \mathbf{x}_i^\top \mathbf{A}^{-1}\Delta\mathbf{x}_i$. When $\Delta \mathbf{X}$ is orthonormal, it follows that

$$
\Delta\mathbf{X}^{-1} = \Delta\mathbf{X}^\top.
$$

\noindent which implies

$$
\mathbf{A}^{-1} = (2\Delta\mathbf{X}\Delta\mathbf{X}^\top)^{-1} = \frac{1}{2}\mathbf{I}_{n \times n} 
$$

\noindent and

$$
s_i^{-1} = \frac{1}{2}.
$$

\noindent Consequently, when $\Delta \mathbf{X}$ is orthonormal, the entirety of Term 1 simplifies to

$$
\begin{gathered}
\frac{1}{2}\mathbf{I}_{n \times n}(\mathbf{I}_{n \times n} - \frac{1}{2} \Delta \mathbf{x}_i \Delta \mathbf{x}_i^\top \frac{1}{2}\mathbf{I}_{n \times n}) = \\
\frac{1}{2}\mathbf{I}_{n \times n} - \frac{1}{8} \Delta \mathbf{x}_i \Delta \mathbf{x}_i^\top.
\end{gathered}
$$

\noindent As expected, Term 1 is a matrix that maps vectors to the subspace of $\mathbb{R}^n$ that is orthogonal to $\Delta \mathbf{x}_i$ (shown in Propositions \ref{prop:orthogonal_m}, \ref{prop:null_space_span_m}, and  \ref{prop:rank_m}).

Next, let us revisit Term 2, as discussed in \S\ref{sec:geometric_interpretation_of_solution}. When $\Delta \mathbf{X}$ is orthonormal, we can simplify this term as follows:

$$
s_i^{-1}\mathbf{A}^{-1}\Delta \mathbf{x}_i \Delta\mathbf{f}_i^\top = \frac{1}{4} \Delta \mathbf{x}_i\Delta \mathbf{f}_i^\top.
$$

\noindent This simplification for Term 2 shows that 

$$\mathbf{A}^{-1}\Delta \mathbf{x}_i = \frac{1}{2}\Delta \mathbf{x}_i$$

\noindent for all $i$, meaning $\Delta \mathbf{x}_i$ is always an Eigenvector of $\mathbf{A}^{-1}$ with an associated Eigenvalue of $\frac{1}{2}$.  Thus, Term 2 is indeed orthogonal to Term 1 as the only component being reintroduced is in the direction of $\Delta \mathbf{x}_i$, thereby addressing the potential misalignment between the two terms as noted in \S\ref{sec:geometric_interpretation_of_term_2}.    
\end{proof}

\vspace{10pt}

Given this analysis, the default structure for the $\Delta \mathbf{X}$ matrix in our algorithm is square and orthonormal, as seen in Algorithm \ref{alg:tangent_bundle_matrix}.  In the evaluation in \S\ref{sec:evaluation}, we empirically demonstrate the advantage of using an orthonormal matrix structure for $\Delta \mathbf{X}$, further bolstering this analysis.  





\subsection{Evaluation 1: Supplemental Conditions}
\label{sec:evaluation1_supplement}

As discussed in \S\ref{sec:conditions}, the main version of Evaluation 1 requires that the conditions remain entirely compatible with the benchmark function implemented in Tensorly, without any modifications, optimizations, or condition-specific code adjustments. While this ensures a fair and uniform comparison, it also imposes significant constraints on approaches that rely on custom code or specialized configurations to achieve optimal performance.  In this section, we aim to lift these constraints, enabling conditions to leverage any necessary setup to maximize performance.

We assess nine conditions in this section:

\begin{enumerate}
    \item Forward-mode automatic differentiation with JAX \cite{jax2018github} backend, JIT-compiled on CPU (abbreviated as \textit{FAD-JAX})
    \item Reverse-mode automatic differentiation with JAX \cite{jax2018github} backend, JIT-compiled on CPU (abbreviated as \textit{RAD-JAX})
    \item Forward-mode automatic differentiation with JAX \cite{jax2018github} backend, JIT-compiled on GPU (abbreviated as \textit{FAD-JAX-GPU})
    \item Reverse-mode automatic differentiation with JAX \cite{jax2018github} backend, JIT-compiled on GPU (abbreviated as \textit{RAD-JAX-GPU})
    \item Forward-mode automatic differentiation with Rust ad-trait \cite{adtrait} backend (abbreviated as \textit{FAD-Rust})
    \item Reverse-mode automatic differentiation with Rust ad-trait \cite{adtrait} backend (abbreviated as \textit{RAD-Rust})
    \item Finite Differencing with Rust ad-trait \cite{adtrait} backend (abbreviated as \textit{FD-Rust})
    \item Web of Affine Spaces Optimization with orthonormal $\Delta \mathbf{X}$ matrix and Rust ad-trait \cite{adtrait} backend (abbreviated as \textit{WASP-O-Rust}).
    \item Web of Affine Spaces Optimization with random, non-orthonormal $\Delta \mathbf{X}$ matrix and and Rust ad-trait \cite{adtrait} backend (abbreviated as \textit{WASP-NO-Rust}).
\end{enumerate}

We follow the same three sub-experiments outlined in \S\ref{sec:subexp1}–\S\ref{sec:subexp3}, using all the same parameters. Additionally, all sub-experiments adhere to the same procedure described in \S\ref{sec:procedure} and employ the metrics defined in \S\ref{sec:metrics}.

All experiments in this section are executed on a Desktop computer with an Intel i9 4.4GHz processor, 32 GB of RAM, and Nvidia RTX 4080 GPU (with CUDA enabled for the JAX GPU conditions).

Results for the supplemental Sub-experiment 1 can be seen in the top row of Figure \ref{fig:se1_supplement}.  The WASP conditions have faster runtime than several other conditions, especially for $n < 500$.  For instance, WASP-O-Rust and WASP-NO-Rust are faster than all JAX conditions over this range, with no practical difference observed between CPU and GPU configurations.  Additionally, WASP-O-Rust and WASP-NO-Rust are faster than FD-Rust and FAD-Rust all the way up to $n=1000$ due to significantly fewer calls to the function.  Also, we see that all conditions, including WASP, are considerably slower than RAD-Rust for nearly all settings of $n$.  Thus, if a high-quality reverse-mode AD is available and feasible to maintain, this will likely lead to favorable efficiency and accuracy outcomes for gradient calculations.  Lastly, the average error results again indicate that the orthonormal structure of $\Delta \mathbf{X}$ is important for stable and more accurate results, as our analyses in \S\ref{sec:tangent_matrix_structure} would suggest.         

Results for the supplemental Sub-experiment 2 can be seen in the middle row of Figure \ref{fig:se1_supplement}.  These results resemble the trends seen in the main version of Sub-experiment 2 in \S\ref{sec:subexp2}.  We see that the WASP conditions are more efficient due to fewer calls to the function, though this efficiency comes at the cost of a small amount of error.  Again, we see that the accuracy is more stable when WASP uses an orthonormal $\Delta \mathbf{X}$.  

Results for the supplemental Sub-experiment 3 can be seen in the bottom row of Figure \ref{fig:se1_supplement}.  Again, these results match the takeaways from \S\ref{sec:subexp3}.  The WASP conditions are highly sensitive to step size.  When the step size becomes too large, the error detection and correction mechanism will almost always trigger another iteration meaning that WASP will essentially revert to a finite differencing strategy.